%% file: main.tex
\documentclass[11pt]{article}
\pdfoutput=1
\usepackage{enumerate}
\usepackage{pdfsync}
\usepackage[OT1]{fontenc}

\usepackage[usenames]{color}
\usepackage{smile}
\usepackage[colorlinks,
            linkcolor=red,
            anchorcolor=blue,
            citecolor=blue
            ]{hyperref}
\usepackage{fullpage}
\usepackage{hyperref}
\usepackage[protrusion=true,expansion=true]{microtype}
\usepackage{pbox}
\usepackage{setspace}
\usepackage{tabularx}
\usepackage{float}
\usepackage{wrapfig,lipsum}
\usepackage{enumitem}
\usepackage{microtype}
\usepackage{graphicx}
\usepackage{subfigure}
\usepackage{enumerate}
\usepackage{enumitem}
\usepackage{pgfplots}
\usetikzlibrary{arrows,shapes,snakes,automata,backgrounds,petri}
\usepackage{booktabs} 

\newtheorem{condition}[theorem]{Condition}

\newcommand{\la}{\langle}
\newcommand{\ra}{\rangle}
\newcommand{\regret}{\textbf{Regret}}
\newcommand{\UHb}{\underline \Hb}
\newcommand{\KL}{\text{KL}}
\newcommand{\var}{\text{Var}}
\newcommand{\adv}{\text{adv}}
\newcommand{\stoch}{\text{stoc}}

\allowdisplaybreaks

\usepackage{enumitem}

\makeatletter
\newcommand*{\rom}[1]{\expandafter\@slowromancap\romannumeral #1@}
\makeatother
\title{\LARGE Optimal Online Generalized Linear Regression with Stochastic Noise and Its Application to Heteroscedastic Bandits\footnote{This is a revised version of the original manuscript titled `Bandit learning with general function classes: Heteroscedastic noise and variance-dependent regret bounds'. In this updated version, we have added new theoretical results on the FTRL algorithm and mainly focused on stochastic online regression. Refer to \url{https://arxiv.org/abs/2202.13603v1} for the previous version, which contains more results on heteroscedastic bandits. }}

\author
{
    Heyang Zhao\thanks{Department of Computer Science, University of California, Los Angeles, CA 90095, USA; e-amil:{\tt hyzhao@cs.ucla.edu}}
    ~~~and~~~
	Dongruo Zhou\thanks{Department of Computer Science, University of California, Los Angeles, CA 90095, USA; e-mail: {\tt drzhou@cs.ucla.edu}} 
	~~~and~~~
    Jiafan He\thanks{Department of Computer Science, University of California, Los Angeles, CA 90095, USA; e-mail: {\tt jiafanhe19@ucla.edu}}
    ~~~and~~~
	Quanquan Gu\thanks{Department of Computer Science, University of California, Los Angeles, CA 90095, USA; e-mail: {\tt qgu@cs.ucla.edu}}
}

\begin{document}
\date{}
\maketitle

\begin{abstract}
    We study the problem of online generalized linear regression in the stochastic setting, where the label is generated from a generalized linear model with possibly unbounded additive noise. We provide a sharp analysis of the classical \emph{follow-the-regularized-leader} (FTRL) algorithm to cope with the label noise. More specifically, for $\sigma$-sub-Gaussian label noise, our analysis provides a regret upper bound of $O(\sigma^2 d \log T) + o(\log T)$, where $d$ is the dimension of the input vector, $T$ is the total number of rounds. We also prove a $\Omega(\sigma^2d\log(T/d))$ lower bound for stochastic online linear regression, which indicates that our upper bound is nearly optimal. In addition, we extend our analysis to a more refined Bernstein noise condition. As an application, we study generalized linear bandits with heteroscedastic noise and propose an algorithm based on FTRL to achieve the first variance-aware regret bound. 
\end{abstract}

\section{Introduction} \label{section:Introduction}

Online learning \citep{cesa2006prediction} plays a crucial role in modern data analytics and machine learning, where a learner progressively interacts with an environment, interactively updates its prediction utilizing sequential data. As a fundamental problem in online learning, online linear regression has been well studied in the adversarial setting \citep{littlestone1991line, azoury2001relative, bartlett2015minimax}. 

In the classic adversarial setting of online linear regression with square loss, the adversary initially generates a sequence of feature vectors $\{\xb_t\}_{t\ge 1}$ in $\RR^d$ with a sequence of labels (i.e., responses) $\{y_t\}_{t \ge 1}$ in $\RR$. At each round $t \ge 1$, $\xb_t$ is revealed to the learner and the learner then makes a prediction $\hat{y}_t \in \RR$ on $\xb_t$. Afterward, the adversary reveals $y_t$,  penalizes the learner by the square loss $(y_t - \hat y_t)^2$, and enters the next round. The goal of the learner is to minimize the total loss of the first $T$ rounds, which is measured by the adversarial regret defined as follows \citep{bartlett2015minimax}: 
    $$\cR^{\adv}(T) := \sum_{t = 1}^T (\hat y_t - y_t)^2 - \inf_{\bmu \in \RR^d} \sum_{t = 1}^T (\la \xb_t, \bmu \ra - y_t)^2, $$
The adversarial regret indicates how far the current predictor is away from the best linear predictor in hindsight. 
Since the labels $\{y_t\}_{t \ge 1}$ are arbitrarily chosen by the adversary in the adversarial setting, existing results on regret upper bound usually require $\{y_t\}$ to be uniformly bounded, i.e., $y_t \in [-Y, Y]$ for all $t\ge 1$ and sometimes $\{\xb_t\}_{t\ge 1}$ is assumed to be known to the learner at the beginning of the learning process \citep[e.g.,][]{bartlett2015minimax}. For adversarial setting, it has been shown that the minimax-optimal regret is of $O(dY^2 \log T)$ \citep{azoury2001relative}, which gives a complete understanding about the statistical complexity. 

It is also interesting to consider a stochastic variant of the classic online linear regression problem where $y_t$ is generated from an underlying linear model with possibly unbounded noise $\epsilon_t$. Under this setting, the stochastic regret, which will be formally introduced in Section \ref{section:Preliminaries}, is defined more intuitively as the `gap' between the predicted label and the underlying linear function $f_{\bmu^*}(\xb_t) = \la \xb_t, \bmu^* \ra$. It is worth noting that this stochastic setting is first studied by \citet{maillard2021stochastic}, for which they studied $\sigma$-sub-Gaussian noise and attained an $\tilde{O}(\sigma^2 d^2)$ high-probability regret bound. However, whether such a bound is tight or improvable remains unknown. Thus, a natural question arises: \emph{what is the optimal regret bound for stochastic online linear regression?}


Beyond the optimality of the existing regret bound, another concern is \emph{whether the analysis for sub-Gaussian noise can be extended to other types of zero-mean noise. }Previous analyses for \emph{online-ridge-regression} and \emph{forward} algorithm provided by \citet{maillard2021stochastic} highly rely on the \emph{self-normalized concentration inequality for vector-valued martingale} \citep[Theorem 1]{abbasi2011improved}, which is for sub-Gaussian random variables. 


In this paper, we simultaneously address the aforementioned questions  for stochastic online generalized linear regression, which admits online linear regression as a special case. We provide a sharp analysis for FTRL and  a nearly matching lower bound in the stochastic setting. 

\subsection{Our Contributions}

In this paper, we make a first attempt on achieving a nearly minimax-optimal regret for \emph{stochastic} online linear regression by a fine-grained analysis of \emph{follow-the-regularized-leader} (FTRL). To show the universality of our analysis, we consider a slightly larger function class, generalized linear class, and our main result on online linear regression is given in the form of a corollary. 

Our contributions are summarized as follows: 

\begin{itemize}
    \item We propose a novel analysis on FTRL for online generalized linear regression with stochastic  noise, which provides an $\tilde{O}(\sigma^2d) + o(\log T)$ regret bound under $\sigma$-sub-Gaussian noise, where $d$ is the dimension of feature vectors, $T$ is the number of rounds. Moreover, for general noise with a variance of $\sigma^2$ (not necessarily to be sub-Gaussian), we prove a more fine-grained upper bound of order $\tilde{O}(A^2B^2 + d\sigma^2) + o(\log T)$, where $A$ is the maximum Euclidean norm of feature vectors, $B$ is the maximum Euclidean norm of $\bmu^*$.
    \item We provide a matching regret lower bound for online linear regression, indicating that our analysis of FTRL is sharp and attains the nearly optimal regret bound in the stochastic setting. To the best of our knowledge, this is the first regret lower bound for online (generalized) linear regression in the stochastic setting.
    \item As an application of our tighter result of FTRL for stochastic online regression, we consider generalized linear bandits with heteroscedastic noise \citep[e.g.,][]{zhou2021nearly, zhang2021variance, dai2022variance}. 
    We propose a novel algorithm MOR-UCB based on FTRL, which achieves an $\tilde{O}(d \sqrt{\sum_{t \in [T]}\text{Var} \ \epsilon_t} + \sqrt{dT})$ regret. This is the first variance-aware regret for generalized linear bandits. 
\end{itemize}

\paragraph{Notation.} 
We denote by $[n]$ the set $\{1,\dots, n\}$. For a vector $\xb\in \RR^d$ and matrix $\bSigma\in \RR^{d\times d}$, a positive semi-definite matrix, we denote by $\|\xb\|_2$ the vector's Euclidean norm and define $\|\xb\|_{\bSigma}=\sqrt{\xb^\top\bSigma\xb}$. For two positive sequences $\{a_n\}$ and $\{b_n\}$ with $n=1,2,\dots$, 
we write $a_n=O(b_n)$ if there exists an absolute constant $C>0$ such that $a_n\leq Cb_n$ holds for all $n\ge 1$ and write $a_n=\Omega(b_n)$ if there exists an absolute constant $C>0$ such that $a_n\geq Cb_n$ holds for all $n\ge 1$. $\tilde O(\cdot)$ is introduced to further hide the polylogarithmic factors. For a random event $\cE$, we denote its indicator by $\mathds{1}(\cE)$. 

\section{Related Work} \label{section:Related-Work}

\noindent\textbf{Online linear regression in adversarial setting.} Online linear regression has long been studied in the setting where the response variables (or labels) are bounded and chosen by an adversary \citep{10.1214/aos/1176348140, littlestone1991line, cesa1996worst, kivinen1997exponentiated, Vovk1997CompetitiveOL, bartlett2015minimax, malek2018horizon}. This problem is initiated by \citet{10.1214/aos/1176348140}, where binary labels and $\ell_1$ constrained parameters are considered. \citet{cesa1996worst} proposed a gradient-descent based algorithm, which gives a regret bound of order $O(\sqrt{T})$ when the hidden vector is $\ell_2$ constrained. \citet{Vovk1997CompetitiveOL} proposed Aggregating Algorithm, achieving $O(Y^2 d \log T)$ regret where $Y$ is the scale of labels and $d$ is the dimension of the feature vectors. \citet{bartlett2015minimax} considered the case where the feature vectors are known to the learner at the start of the game and proposed an exact minimax regret for the problem. Afterwards, \citet{malek2018horizon} generalized the results of \citet{bartlett2015minimax} to the cases where the labels and covariates can be chosen adaptively by the environment. Later on, \citet{gaillard2019uniform} proposed Forward algorithm, showing that Foward algorithm without regularization can ahcieve optimal asymptotic regret bound uniform over bounded observations. 

\noindent\textbf{Stochastic online linear regression.} Recently, \citet{maillard2021stochastic} considered the stochastic setting where the response variables are unbounded and revealed by the environment with additional random noise on the true labels. \citet{maillard2021stochastic} discussed the limitations of online learning algorithms in the adversarial setting and further advocated for the need of  complementary analyses for existing algorithms under the stochastic unbounded setting. In their paper, new analyses for online ridge regression and Forward algorithm are proposed, achieving asymptotic $O(\sigma^2 d^2 \log T)$ regret bound.\footnote{We noticed that \citet{maillard2021stochastic} also mentioned a way to acquire a tighter regret bound by applying confident sets with an $O\left(\sqrt{d \log\log T + \log(1 / \delta)}\right)$ radius \citep{tirinzoni2020asymptotically}, where the previous $O\left(\sqrt{d \log(T / \delta)}\right)$ one proposed by \citet{abbasi2011improved} leads to an $\tilde{O}(\sigma^2 d^2)$. In this way, the $T$ dependence in their bound can be further improved. However, the $d$ dependence is still quadratic. }

\noindent\textbf{Learning heteroscedastic bandits.} Heteroscedastic noise has been studied in many settings such as active learning \cite{Antos2010ActiveLI}, regression \citep{aitken1936iv, goldberg1997regression, Chaudhuri2017ActiveHR, kersting2007most}, principal component analysis \citep{hong2016towards, hong2018optimally} and Bayesian optimization \citep{assael2014heteroscedastic}. However, only a few works have considered heteroscedastic noise in bandit settings. \citet{cowan2015normal} considered a variant of multi-armed bandits where the noise at each round is a Gaussian random variable with unknown variance. \citet{kirschner2018information} is the first to formally introduce the concept of stochastic bandits with heteroscedastic noise. In their model, the variance of the noise at each round $t$ is a function of the evaluation point $x_t$, $\rho_t = \rho(x_t)$, and they further assume that the noise is $\rho_t$-sub-Gaussian. $\rho_t$ can either be observed at time $t$ or either be estimated from the observations. \citet{zhou2021nearly} and \citet{zhou2022computationally} considered linear bandits with heteroscedastic noise and generalized the heteroscedastic noise setting in \citet{kirschner2018information} in the sense that they no longer assume the noise to be $\rho_t$-sub-Gaussian, but only requires the variance of noise to be upper bounded by $\rho_t^2$ and the variances are arbitrarily decided by the environment, which is not necessarily a function of the evaluation point. In the same setting as in \citet{zhou2021nearly}, \citet{zhang2021variance} further considered a strictly harder setting where the noise has unknown variance. They proposed an algorithm that can deal with unknown variance through a computationally inefficient clip technique. Our work basically considers the noise setting proposed by \citet{zhou2021nearly} and further generalizes their setting to bandits with general function classes. We will consider extending it to the harder setting as \citet{zhang2021variance} as future work.

\input{Preliminaries.tex}

\section{Optimal Stochastic Online Generalized Linear Regression} \label{sec:4}

We provide analyses on \emph{follow-the-regularized-leader} (FTRL) in this section. We apply a quadratic regularization in FTRL, as shown in Algorithm \ref{alg:FTRL}. At each round, FTRL aims to predict the unknown $\bmu^*$ defined in Section \ref{nnnn}. To achieve this goal, FTRL computes the minimizer of the regularized cumulative loss over all previously observed context $\xb_t$ and label $y_t$. This is slightly different from what we want to do for the adversarial setting, where the goal of FTRL is to predict the best predictor over $T$ rounds.

\begin{figure} 
\begin{minipage}{0.54\textwidth}
\begin{algorithm}[H]
    \caption{Follow The Regularized Leader (FTRL)} \label{alg:FTRL}
    \begin{algorithmic}[1]
        \STATE \textbf{Input: } $\cG, \lambda$. 
        \STATE \textbf{Initialize: } $\hat\bmu_1 \leftarrow \zero$.  
        \FOR {$t \ge 1$}
        \STATE Observe $\xb_t$. 
        \STATE Output $\hat y_t = f_{\hat \bmu_t}(\xb_t)$. 
        \STATE Observe $y_t$. 
        \STATE Update $\hat \bmu_{t + 1} \leftarrow \argmin\limits_{\bmu \in \RR^d} \lambda \|\bmu\|_2^2 + \sum\limits_{\tau = 1}^{t} \ell_\tau(\bmu), $where $\ell_\tau$ is defined in \eqref{eq:def:loss}. 
        \ENDFOR
    \end{algorithmic}
\end{algorithm}
\end{minipage}
\begin{minipage}{0.45\textwidth}
    \includegraphics*[width=\textwidth]{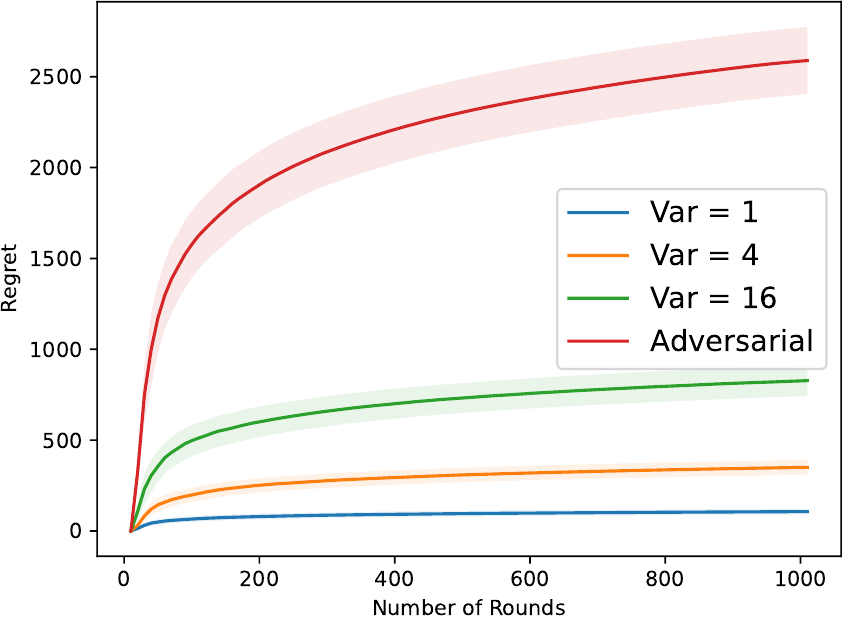}
    \caption{Regret of FTRL under different types of labels.  \label{figure:2}}
\end{minipage}
\end{figure}

\subsection{Regret Upper Bound}
We first propose the stochastic regret upper bound for FTRL under the stochastic online regression setting. 
\begin{theorem}[Regret of FTRL] \label{thm:ftrl}
    Set $\lambda = 4A^2K^2/\kappa$ and assume that the noise $\epsilon_t$ satisfies Condition~\ref{cond:noise:1} at all rounds $t\in [T]$, then with probability at least $1 - 2\delta$, the regret of Algorithm \ref{alg:FTRL} for the first $T$ rounds is bounded as follows:
    \begin{align} 
        \cR^{\stoch}(T) &\le 34 \kappa^{-1} \cdot \sigma^2 d \log \frac{4dK^2 + T \kappa^2}{4dK^2} + 8\frac{K^2A^2B^2}{\kappa} 
         + 2\left(\frac{2}{\kappa} + \frac{3\kappa}{K^2}\right) \sigma^2 \log(1 / \delta). \notag
    \end{align}
\end{theorem}

Directly applying this theorem to online linear regression, we have the following result. 

\begin{corollary}[Regret of FTRL in stochastic online linear regression]\label{coro:1}
    Suppose that $\phi$ is the identical mapping. Set $\lambda = 4A^2$ and assume that Condition \ref{cond:noise:1} holds. With probability at least $1 - 2\delta$, the regret of Algorithm \ref{alg:FTRL} for the first $T$ rounds is bounded as: \begin{align*} 
        \cR^{\stoch}(T) \le O\left(\sigma^2d\log T + A^2B^2\right). 
    \end{align*} 
\end{corollary}

\begin{remark} 
    Corollary \ref{coro:1} suggests a regret upper bound for stochastic online linear regression with square loss. Recent work by \citet{maillard2021stochastic} studied this stochastic setting and managed to get rid of the $O(A^2B^2d\log T)$ term in classic result for online linear regression in the adversarial setting. \citet{maillard2021stochastic} derived a high probability regret bound of $\tilde{O}(\sigma^2d^2)$ after omitting the $o(\log(T)^2)$ terms (Theorem 3.3, \citealt{maillard2021stochastic}). Unlike their result, our result does not suffer from the quadratic dependence on $d$. As for the $O(A^2B^2)$ term in our result, it is not hard to see that this part of loss is inevitable, since at the first round, the algorithm has no prior knowledge of $\btheta^*$. We defer the detailed analysis on the lower bound of the problem to the next section. 
\end{remark}

\subsection{Experimental Results} \label{sec:experiment1}

\begin{figure*} 
    \begin{center}
    \subfigure[$T = 1000$]{
    \includegraphics*[width=0.32\textwidth]{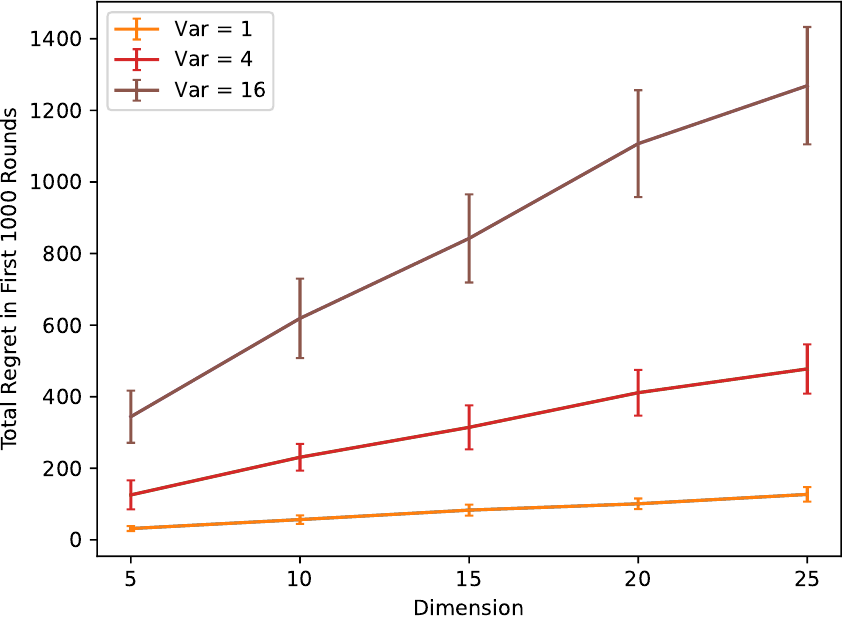}}
    \subfigure[$T = 2000$]{
    \includegraphics*[width=0.32\textwidth]{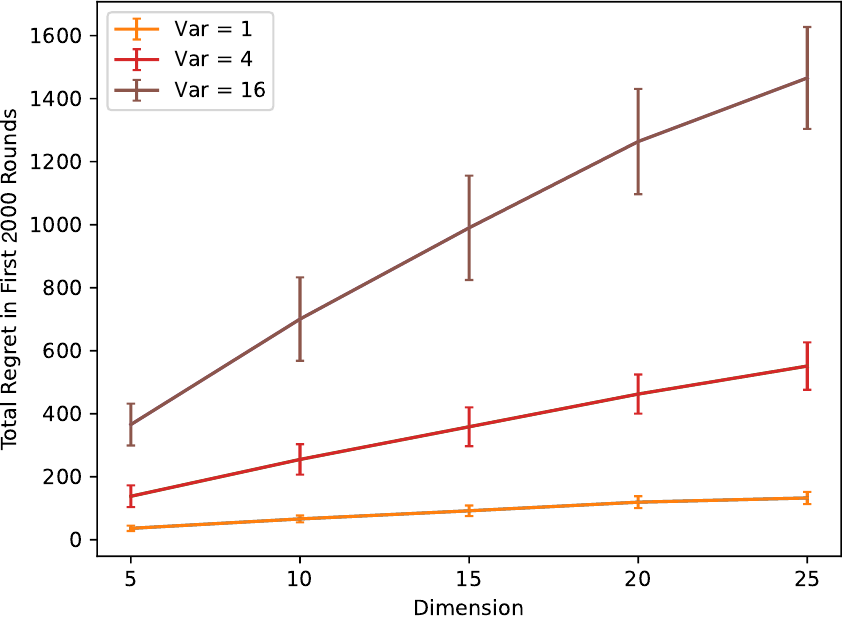}}
    \subfigure[$T = 5000$]{
    \includegraphics*[width=0.32\textwidth]{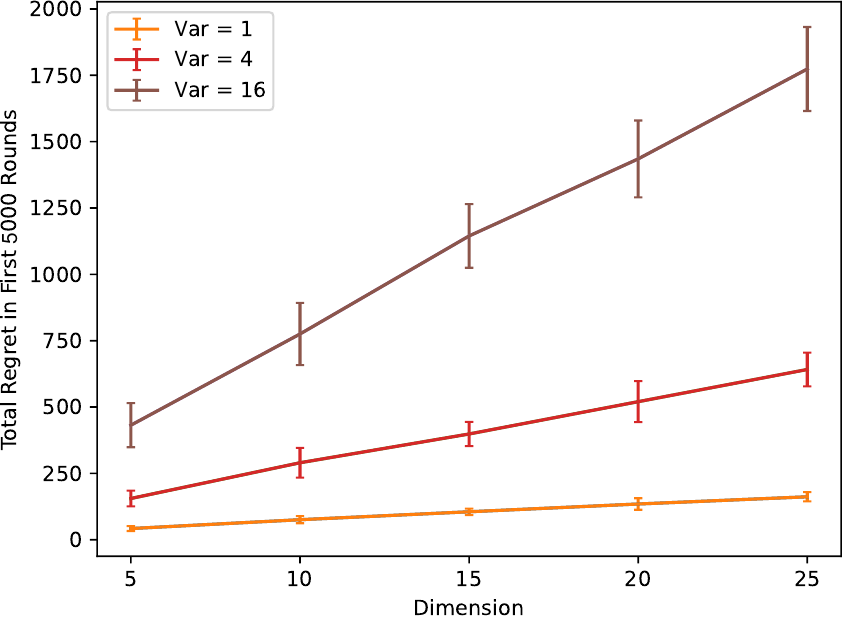}}
    \caption{Cumulative regret of FTRL for online linear regression with stochastic noise under different noise variances. The x-axis is the  dimension of the feature vector.  \label{figure:3}}
    \end{center}
\end{figure*}

In this subsection, we provide experimental evidence supporting that the high-probability regret of FTRL grows linearly on the dimension of the feature vectors. 
We plot the stochastic regret with respect to the dimension of the feature vectors in Figure \ref{figure:3}. More specifically, we compute the stochastic regret over different numbers of rounds, i.e., $\cR^\stoch(1000), \cR^\stoch(2000), \cR^\stoch(5000)$, under different Gaussian noise with $\sigma^2 = 1, 4, 16$. In each trial we sample $\bmu^*$ uniformly from $[-1 / \sqrt{d}, 1 / \sqrt{d}]^d$. At each round, we sample a feature vector uniformly from the unit sphere centered at the origin. 

In Figure \ref{figure:3}, we observe that under a fixed number of rounds, the value of stochastic regret has a linear dependence on $d$, which corroborates our theoretical analysis in Section \ref{sec:4}. 

\subsection{Extension to Bernstein's Condition} \label{subsec:bernstein}

In some real-world scenarios, however, the sub-Gaussianity condition (Condition \ref{cond:noise:1}) may be too strong for general zero-mean noise. 

In this subsection, we consider another assumption that the variance of $\epsilon_t$ is not larger than $\sigma^2$, which is formally stated in Condition \ref{cond:noise:2}. To remove the sub-Gaussianity condition, we introduce a new parameter $R$ in Condition \ref{cond:noise:2}, serving as a large uniform upper bound on the noise $\{\epsilon_t\}_{t\in[T]}$. 

\begin{theorem}[Regret of FTRL] \label{thm:ftrl2}
    Set $\lambda = 4A^2K^2/\kappa$ and assume that the noise $\epsilon_t$ satisfy Condition \ref{cond:noise:2} at all rounds $t\in [T]$, then with probability at least $1 - 2\delta$, the regret of Algorithm \ref{alg:FTRL} for the first $T$ rounds is bounded as follows:
    \begin{align} 
        \cR^{\stoch}(T) &\le 6 \kappa^{-1} \cdot \sigma^2 d \log \frac{4dK^2 + T \kappa^2}{4dK^2} + 8\frac{K^2A^2B^2}{\kappa} + 2\left(\frac{2}{\kappa} + \frac{5\kappa}{24K^2}\right) R^2 \log(1 / \delta). \notag
    \end{align}
\end{theorem}

\begin{remark} 
    When $T$ is sufficiently large, this bound becomes $O(\sigma^2 d \log T)$ in stochastic online linear regression, where $\sigma$ is a uniform bound on $\EE[\epsilon_t^2]$ for $t \in [T]$. Compared with Theorem \ref{thm:ftrl}, this theorem deals with a wider class of noise types. We defer the proof of this theorem to Section~\ref{section:proof:2}. 
\end{remark}

\subsection{Proof Outline of Theorem \ref{thm:ftrl}}
We provide the proof sketch for Theorem \ref{thm:ftrl} here. We introduce the following concepts before presenting our results and analyses. First, we define the cumulative loss as follows: \begin{align} 
    \cL_t(\bmu) = \sum_{\tau = 1}^t \ell_\tau(\bmu) + \lambda\|\bmu\|_2^2. \label{eq:def:cumloss}
\end{align}
For simplicity, we denote the Hessian matrix of $\cL_t$ by $\Hb_t$ for each $t \ge 1$, i.e., 
\begin{align} 
    \Hb_t(\bmu) := \Hb_{\cL_t}(\bmu) = 2\lambda \cdot \Ib + \sum_{\tau = 1}^t \phi'(\xb_\tau^\top \bmu) \cdot \xb_\tau \xb_\tau^\top. \label{eq:def:hessian}
\end{align}
We also construct the following sequence $\{\underline{\Hb}_t\}_{t \in [H]}$: \begin{align} 
    \underline\Hb_t := 2\lambda \cdot \Ib + \kappa \sum_{\tau = 1}^t  \xb_\tau \xb_\tau^\top. \label{eq:def:hessian1}
\end{align}
By the convexity of $\phi$, it is easy to show that $\{\underline{\Hb}_t\}_{t \in [H]}$ is a lower bound of $\{{\Hb}_t(\cdot)\}_{t \in [H]}$ since for all $t \in [T]$, $\underline \Hb_t \preceq \Hb_t(\cdot)$. 

\subsubsection{Regret Decomposition}

We first point out the key place which prevents \citet{maillard2021stochastic} obtaining a tight regret bound. 
In the previous analysis of online learning algorithms in stochastic setting \citep{maillard2021stochastic}, the regret is bounded through the summation of instantaneous regret \begin{align*} \sum_{t = 1}^T \left(\ell_t(\hat\bmu_t) - \ell_t(\bmu^*)\right) \le \sum_{t = 1}^T O\left(\|\hat\bmu_t - \bmu^*\|_{\Gb_t}^2\right) \cdot \|\xb_t\|_{\Gb_t^{-1}}^2 \end{align*} where $\Gb_t = \lambda + \sum_{\tau = 1}^{t - 1} \xb_{\tau} \xb_\tau^\top$ is the sample covariance matrix. Applying a similar confidence ellipsoid as proposed in \citet{abbasi2011improved}, $\|\hat\bmu_t - \bmu^*\|_{\Gb_t}$ is uniformly bounded by $\tilde{O}(\sigma^2d)$. By Matrix Potential Lemma \citep{abbasi2011improved}, $\sum_{t = 1}^T \|\xb_t\|_{\Gb_t^{-1}}^2$ is bounded by $O(d \log T)$. Thus, a quadratic dependence of $d$ is inevitable in their regret upper bound. 

To circumvent this issue, we prove the following lemma, which decomposes the cumulative regret into three terms. 

\begin{lemma}[Regret decomposition] \label{lemma:decomposition}
    For each $t \in [T]$, let $\cL_t$ be the cumulative loss function defined in \eqref{eq:def:cumloss} and $\Hb_t$ be the corresponding Hessian matrix as shown in \eqref{eq:def:hessian}. There exists a sequence $\{\bmu'_t\}_{t \in [T]}$ in $\RR^d$ such that the stochastic regret of Algorithm \ref{alg:FTRL} can be decomposed as follows: \begin{align} 
        \cR^{\stoch}(T) &\le \lambda B^2  + \frac{1}{2} \sum_{t = 1}^T \epsilon_t^2 \|\xb_t\|_{\Hb_{t}^{-1}(\bmu_t')}^2  + \frac{1}{2} \sum_{t = 1}^T \left(\phi(\xb_t^\top \hat\bmu_t) - \phi(\xb_t^\top \bmu^*)\right)^2\|\xb_t\|_{\Hb_{t}^{-1}(\bmu_t')}^2. \notag
    \end{align}
\end{lemma}

Intuitively speaking, $\sum_{t = 1}^T \epsilon_t^2 \|\xb_t\|_{\Hb_{t}^{-1}(\bmu_t')}^2$ represents the part of regret caused by the random noise, while $\sum_{t = 1}^T \left(\phi(\xb_t^\top \hat\bmu_t) - \phi(\xb_t^\top \bmu^*)\right)^2\|\xb_t\|_{\Hb_{t}^{-1}(\bmu_t')}^2$ represents the gap between the estimator $\hat\bmu_t$ and the hidden vector $\bmu^*$. We bound these two terms separately as follows.

\subsubsection{Bounding the estimation error}

To derive a high-probability upper bound for the term $\sum_{t = 1}^T \left(\phi(\xb_t^\top \hat\bmu_t) - \phi(\xb_t^\top \bmu^*)\right)^2\|\xb_t\|_{\Hb_{t}^{-1}(\bmu_t')}^2$ in Lemma \ref{lemma:decomposition}, we start by considering the connection between $\left(\phi(\xb_t^\top \hat\bmu_t) - \phi(\xb_t^\top \bmu^*)\right)$ and the cumulative regression regret. 

\begin{lemma}[Connection between squared estimation error and regret] \label{lemma:connect1}
    Consider an arbitrary online learner interactively trained with stochastic data for $T$ rounds as described in Section \ref{section:Preliminaries}. If Condition \ref{cond:noise:1} is true for the noise at all the rounds $t \in [T]$, then the following inequality holds with probability at least $1 - \delta$: \begin{align*} 
        \sum_{t = 1}^T \left[\xb_t^\top (\hat\bmu_t - \bmu^*)\right]^2 \le \frac{4}{\kappa} \cR^{\stoch}(T) + \frac{16}{\kappa^2} \cdot \sigma^2 \log(1 / \delta). 
    \end{align*}
\end{lemma}

\subsubsection{Bounding the weighted sum of squared noise}

According to Lemma \ref{lemma:decomposition}, it remains to bound the term $\sum_{t = 1}^T \epsilon_t^2 \|\xb_t\|_{\Hb_{t}^{-1}(\bmu_t')}^2$, which can be regarded as the weighted sum of squared sub-Gaussian random variables. 

In our analysis, we first show that $\epsilon_t^2 \|\xb_t\|_{\Hb_{t}^{-1}(\bmu_t')}^2$ is a sub-exponential random variable and apply a tail bound on the total summation. The result is presented in the following lemma. 

\begin{lemma} \label{lemma:noise-sum}
    Suppose that the sequence of noise $\{\epsilon_t\}_{t\in [T]}$ satisfies Condition \ref{cond:noise:1}. For each $t \in [T]$, let $\underline \Hb_t$ be the matrix defined in \eqref{eq:def:hessian1}. With probability at least $1 - \delta$, \begin{align} 
        \sum_{t = 1}^T \epsilon_t^2 \|\xb_t\|_{\UHb_t^{-1}}^2 &\le 34\kappa^{-1} \cdot \sigma^2 d \log \frac{d \lambda + T \kappa A^2}{d \lambda} + 24 \sigma^2 \cdot \frac{A^2}{\lambda} \log(1 / \delta). \notag 
    \end{align}
\end{lemma}

\noindent\textbf{Putting all things together.} From Lemmas \ref{lemma:decomposition}, \ref{lemma:connect1}, \ref{lemma:noise-sum}, we conclude by a union bound that, with probability at least $1 - 2\delta$, \begin{align*} 
    \cR^{\stoch}(T) \le O(\kappa^{-1}\cdot \sigma^2 d\log T) + o(\log T). 
\end{align*}

\section{Lower Bound}

In this section, we present a lower bound of the stochastic regret for stochastic online linear regression, which indicates the FTRL algorithm is already tight and optimal.

\begin{theorem}[Lower bound for stochastic online linear regression] \label{thm:lower}
    Consider the case where $\kappa = K = 1$ in Assumption \ref{cond:activation}. Then the problem degrades to stochastic online linear regression where the stochastic regret can be defined as follows: \begin{align} 
        \cR^{\stoch}(T) &:= \sum_{t=1}^T \ell_t(\hat \bmu_t) - \sum_{t=1}^T \ell_t(\bmu^*)\notag 
    \end{align}
    Suppose that the noise sequence $\{\epsilon_t\}_{t \in [T]}$ is a sequence of i.i.d. Gaussian random variables, i.e., $\epsilon_t \sim \cN(0, \sigma^2)$ for all $t \in [T]$. When $T$ is sufficiently large, for any online regression algorithm, there exists $\bmu^* \in \RR^d$ and a sequence of feature vectors $\{\xb_t\}_{t\in [T]}$ such that $\EE[\cR^{\stoch}(T)] \ge \Omega(\sigma^2d \log\left(T / d\right) + B \max_{t \in [T]} \|\xb_t\|_2)$. 
\end{theorem}

We notice that \citet{mourtada2022exact} provided a lower bound for expected
excess risk in a random-design linear prediction problem, which can be seen as an offline version of the considered problem in our work. However, in online linear regression, we have to prove the existence of a $\btheta^*$ which makes their result hold for every round $t \in[T]$. 

The proof of Theorem \ref{thm:lower} involves the application of Pinsker's inequality, which is stated in Lemma \ref{lemma:pinsker}.

\begin{proof}[Proof of Theorem \ref{thm:lower}]
For $\bmu \in \RR^d$, $\|\bmu\|_2 \le B$, we denote by $P_{\bmu}$ the measure on $\hat y_1, \cdots, \hat y_T, y_1, \cdots, y_T$ generated by the interaction between the algorithm and the environment. 

We suppose that the sequence of feature vectors are fixed and consists of unit vectors. Let $\cT_i := \{t \in [T]| \xb_t = \eb_i\}$ for each $i \in [d]$, and $t_{i, j}$ be the $j$-th element in $\cT_i$. 

Assume that $\bmu^*$ is uniformly sampled from $[-B / \sqrt{d}, B / \sqrt{d}]^d$ at the beginning of the first round. 

We show that for all $i \in [d]$, $\EE[(\hat y_{t_{i, j}} - \la \xb_{t_{i, j}}, \bmu^*\ra)^2] \ge \Omega\left(\frac{\sigma^2}{j}\right)$ for all sufficiently large $j$. Consider any pair of $\bmu^1, \bmu^2 \in [-B / \sqrt{d}, B / \sqrt{d}]^d$ such that $\bmu^1_k = \bmu^2_k$ for all $k \in [d] \backslash \{i\}$ and $\left|\bmu^1_i - \bmu^2_i\right| \in \left[\frac{\sigma}{8\cdot \sqrt{j - 1}}, \frac{\sigma}{4\cdot \sqrt{j - 1}}\right]$. By Lemma \ref{lemma:pinsker}, for any event $A$ in the filtration generated by the labels before round $t_{i, j}$, \begin{align*} 
    \left|P_{\bmu^1}(A) - P_{\bmu^2}(A)\right| \le \sqrt{\frac{1}{2} \KL\left(P_{\bmu^1}\| P_{\bmu^2}\right)} = \sqrt{\frac{1}{2} \sum_{t = 1}^{t_{i, j} - 1} \frac{\left(\xb_t^\top \bmu^1 - \xb_t^\top \bmu^2\right)^2}{2\sigma^2}} \le \frac{1}{8}. 
\end{align*}
Thus, for any event $A$, $P_{\bmu^1}(A) + P_{\bmu^2}(\overline A) \ge 7 / 8$. Let $A = \{\hat y_{t_{i, j}} \ge (\bmu_i^1 + \bmu_i^2) / 2\}$. We have \begin{align} &\EE_{\bmu^1}\left[(\hat y_{t_{i, j}} - \la \xb_{t_{i, j}}, \bmu^1\ra)^2\right] + \EE_{\bmu^2}\left[(\hat y_{t_{i, j}} - \la \xb_{t_{i, j}}, \bmu^2\ra)^2\right] \notag \\&= \EE_{\bmu^1}\left[(\bmu_1^i - \hat y_{t_{i, j}})^2\right] + \EE_{\bmu^2}\left[(\bmu_2^i - \hat y_{t_{i, j}})^2\right]\notag \\&\ge (P_{\bmu^1}(A) + P_{\bmu^2}(\overline A)) \frac{\sigma^2}{256(j - 1)} \ge \Omega(\sigma^2 / j) \label{eq:lower:pair}\end{align}
For any `segment' $\cS = \{\bmu \in [-B / \sqrt{d}, B / \sqrt{d}]^d | \bmu_i \in \left(a, a + \frac{\sigma}{4\cdot \sqrt{j - 1}}\right), \bmu_k = c_k\ for\ k \neq i\}$, we denote $\bmu_{\cS}(u) := (c_1, \cdots, a + u, c_{i + 1}, \cdots, c_d)^\top$. 

We have \begin{align*} 
    &\EE_{\bmu \in \cS}\left[(\hat y_{t_{i, j}} - \la \xb_{t_{i, j}}, \bmu\ra)^2\right] \\&=\frac{4\sqrt{j - 1}}{\sigma} \int_{a}^{a + \frac{\sigma}{4\cdot \sqrt{j - 1}}}\EE_{\bmu_{\cS}(u)}\left[(\hat y_{t_{i, j}} - (\bmu_\cS(u))_i)^2\right] \text{d}u
    \\&\ge \Omega(\sigma^2 / j), 
\end{align*} where the last inequality holds due to \eqref{eq:lower:pair}. 

From the arbitrariness of $\{c_k\}$ in $\cS$ and the uniform distribution of $\bmu^*$, we have the following conclusion for a `slice' in the cube  $[-B / \sqrt{d}, B / \sqrt{d}]^d$: 

For $\cB = \left\{\bmu \in [-B / \sqrt{d}, B / \sqrt{d}]^d | \bmu_i \in \left(a, a + \frac{\sigma}{4\cdot \sqrt{j - 1}}\right)\right\}$, we have \begin{align*} 
    \EE_{\bmu \in \cB}\left[(\hat y_{t_{i, j}} - \la \xb_{t_{i, j}}, \bmu\ra)^2\right] \ge \Omega(\sigma^2 / j). 
\end{align*}
It is then straightforward to conclude that $\EE\left[(\hat y_{t_{i, j}} - \la \xb_{t_{i, j}}, \bmu\ra)^2\right] \ge \Omega(\sigma^2 / j)$ when $j$ is sufficiently large so that $\frac{\sigma}{4\cdot \sqrt{j - 1}}$ is sufficiently small. 

If we generate $\{\xb_t\}$ such that $\cT_i = \lfloor T / d \rfloor$ for all $i \in [d]$, then we have \begin{align*} 
    \sum_{t = 1}^T \EE\left[\left(\hat y_t - \la\xb_t, \bmu^*\ra\right)^2\right] \ge \Omega(\sigma^2 \log(T / d)). 
\end{align*}
Trivially, \begin{align*} 
    \sum_{t = 1}^T \EE\left[\left(\hat y_t - \la\xb_t, \bmu^*\ra\right)^2\right] \ge \sum_{i \in [d]} \EE\left[\left(\hat y_{t_{i, 1}} - \bmu^*_i\right)^2\right]
    \ge \sum_{i \in [d]} \EE\left[(\bmu^*_i)^2\right] \ge \Omega(B). 
\end{align*}
Applying Lemma \ref{lemma:connect1}, we can further conclude that $\EE[\cR^{\stoch}(T)] \ge \Omega(\sigma^2d \log\left(T / d\right) + B)$ when $T$ is sufficiently large. 

\end{proof}

\section{Application to Heteroscedastic Bandits}\label{section:bandits}

In the last section, it is shown that FTRL achieves nearly optimal regret faced with sequential data with zero-mean noise. Particularly, in Section \ref{subsec:bernstein}, we show that FTRL is capable of dealing with noise such that $\sigma^2 = \max_{t \in [T]} \var[\epsilon_t]$. However, it only utilizes the max variance information, which is not satisfactory if we want to see the trend of the change of the variance w.r.t. rounds. Thus, it is natural to ask whether we can design an algorithm for the generalized linear bandits, whose statistical complexity depends on the variance \emph{adaptively}, says, depends on the \emph{total variance} $\Sigma_t \var[\epsilon_t]$. 
For linear bandits setting, such a goal has been achieved in \citet{zhou2021nearly, zhang2021variance}. 

\subsection{Problem Setup}

We consider a heteroscedastic variant of the classic stochastic bandit problem with generalized linear reward functions. At each round $t \in [T]$ ($T \in \NN$), the agent observes a decision set $\cD_t \subseteq \RR^d$ which is chosen by the environment. The agent then selects an action $\ab_t \in \cD_t$ and observes reward $r_t$ together with a corresponding variance upper bound $\sigma_t^2$. We assume that $r_t = f^*(\ab_t) + \epsilon_t$ where $f^* = f_{\btheta^*} \in \cG$ defined in \eqref{eq:def:glm} is the underlying real-valued reward function and $\epsilon_t$ is a random noise. We make the following assumption on $\epsilon_t$. 
\begin{assumption}[Heteroscedastic noise]\label{eq:noise}
    The noise sequence $\{\epsilon_t\}_{t \in [T]}$ is a sequence of independent zero-mean random variables such that \begin{align*} 
        \forall t \ge 1, \quad \PP\left(|\epsilon_t| \le R\right) = 1, \EE\left[\epsilon_t^2\right] \le \sigma_t^2. 
    \end{align*}
\end{assumption}
The goal of the agent is to minimize the following cumulative regret: \begin{align} 
    \regret(T) := \textstyle{\sum_{t = 1}^T} [f^*(\ab_t^*) - f^*(\ab_t)], \label{def:regret}
\end{align} where the optimal action $\ab_t^*$ at round $t \in [T]$ is defined as $
    \ab_t^* := \argmax_{\ab \in \cD_t} f^*(\ab)
$. 

\subsection{The Proposed Algorithm}

\begin{algorithm}[t!]
    \caption{\textbf{M}ulti-layer \textbf{O}nline \textbf{R}egression-UCB} \label{alg:1}
    \begin{algorithmic}[1]
        \STATE \textbf{Input: } $T, \lambda, R, \overline{\sigma} > 0$. 
        \STATE \textbf{Initialize: } Set $L \leftarrow \lceil \log_2 R / \overline{\sigma} \rceil$ \\ and $\cC_{1, l} \leftarrow \cF$, $\Psi_{1, l} \leftarrow \varnothing$ for all $l \in [L]$. 
        \FOR {$t = 1 \cdots T$}
        \STATE Observes $\cD_t$. 
        \STATE Choose $\ab_t \leftarrow \argmax\limits_{\ab \in \cD_t} \min\limits_{l \in [L]} \phi\left(\hat \btheta_{t, l}^\top \ab + \beta_{t, l} \|\ab\|_{\bSigma_{t, l}^{-1}}\right)$. \label{alg1:line:5}
        \STATE Observe stochastic reward $r_t$ and $\sigma_t^2$. \label{alg1:line:6}
        \STATE Find $l_t$ such that $2^{l_t + 1} \overline{\sigma} \ge \max(\overline{\sigma}, \sigma_t) \ge 2^{l_t} \overline{\sigma}$. \label{alg1:line:7}
        \STATE Update $\Psi_{t + 1, l_t} \leftarrow \Psi_{t, l_t} \cup \{t\}$ \\ and $\Psi_{t + 1, l} \leftarrow \Psi_{t, l}$ for all $l \in [L] \backslash \{l_t\}$. \label{alg1:line:8}
        \STATE Compute $\hat \btheta_{t + 1, l} \leftarrow \argmin\limits_{\btheta \in \RR^d} \lambda \|\btheta\|_2^2 + \sum\limits_{\tau \in \Psi_{t + 1, l}} \ell_\tau(\btheta) $ for all $l \in [L]$, where $\ell_\tau$ is defined following \eqref{eq:def:loss}: \begin{align} 
            \ell_{\tau}(\btheta) := -\ab_\tau^\top \btheta r_\tau + \int_0^{\ab_\tau^\top \btheta} \phi(z) \text{d}z. \notag
        \end{align}
        \STATE Compute $\bSigma_{t + 1, l} \leftarrow 2\lambda \Ib + \sum_{\tau \in \Psi_{t + 1, l}} \ab_\tau \ab_\tau^\top$ for all $l \in [L]$. 
        \ENDFOR
    \end{algorithmic}
\end{algorithm}

\noindent\textbf{Existing approach.} To tackle the heteroscedastic bandit problem, for the case where the $\cF$ is the linear function class (i.e., $f(a) = \la\btheta^*, a\ra$ for some $\btheta^* \in \RR^d$), a \emph{weighted linear regression} framework \citep{kirschner2018information, zhou2021nearly} has been proposed. Generally speaking, at each round $t \in [T]$, weighted linear regression constructs a confidence set $\cC_t$ based on the empirical risk minimizarion (ERM) for all previous observed actions $a_s$ and rewards $r_s$ as follows:
\begin{small}
\begin{align}
    &\btheta_t \leftarrow \argmin_{\btheta \in \RR^d}\lambda\|\btheta\|_2^2 +  \textstyle{\sum_{s \in [t]} w_s(\la \btheta, a_s\ra - r_s)^2},\ \textstyle{\cC_t\leftarrow  \big\{\btheta \in \RR^d\big|\sum_{s = 1}^t w_s (\la \btheta, a_s\ra - \la \btheta_t, a_s\ra)^2 \leq \beta_t \big\}},\notag
\end{align}
\end{small}
where $w_s$ is the weight, and $\beta_t, \lambda$ are some parameters to be specified. $w_s$ is selected in the order of the inverse of the variance $\sigma_s^2$ at round $s$ to let the variance of the rescaled reward $\sqrt{w_s}r_s$ upper bounded by 1. Therefore, after the weighting step, one can regard the heteroscedastic bandits problem as a homoscedastic bandits problem and apply existing theoretical results 
to it. To deal with the general function case, a direct attempt is to replace the $\la \btheta, a\ra$ appearing in above construction rules with $f(a)$. However, such an approach requires that $\cF$ is \emph{close} under the linear mapping, which does not hold for general function class $\cF$. 

We propose our algorithm MOR-UCB as displayed in Algorithm \ref{alg:1}. At the core of our design is the idea of partitioning the observed data into several layers and `packing' data with similar variance upper bounds into the same layer as shown in line 7-8 of Algorithm \ref{alg:1}. Specifically, for any two data belonging to the same layer, their variance will be at most one time larger than the other. Next in line 9, our algorithm implements FTRL to estimate $f^*$ according to the data points in $\Psi_{t + 1, l}$. 
Then in line 5, the agent makes use of $L$ confidence sets simultaneously to select an action based on the \emph{optimism-in-the-face-of-uncertainty} (OFU) principle over all $L$ number of levels.


\subsection{Theoretical Results}
We provide the theoretical guarantee of MOR-UCB here.
\begin{theorem}[Cumulative regret for generalized linear bandits]
    Suppose that $\|\btheta^*\|_2 \le 1$ and for all $\ab \in \bigcup_{t \in [T]} \cD_t$, $\|\ab\|_2 \le 1$. 
    Set $\lambda = 4K^2 / \kappa$ and \begin{align} 
        \beta_{t, l} &= 16\cdot 2^{l} \overline \sigma\kappa^{-1 / 2} \sqrt{d \log \left(\frac{2d\lambda + t\kappa A^2}{2d\lambda}\right) \log \left(4t^2 L/ \delta\right)} \notag \\&\quad + 4R \cdot \kappa^{-1 / 2} \log(4t^2 L/ \delta) + 2\sqrt{2 / \kappa} K \notag
    \end{align} in Algorithm \ref{alg:1}. With probability at least $1 - \delta$, the regret of Algorithm \ref{alg:1} in the first $T$ rounds satisfies that: \begin{align*} 
        \regret(T) \le \tilde{O}\left(\frac{K}{\kappa} d \sqrt{\sum_{t = 1}^T\sigma_t^2} + \left(R + K\right) K \cdot \kappa^{-1} \sqrt{dT}\right)
    \end{align*}
\end{theorem}

\begin{remark} 
In the case of heteroscedastic linear bandits \citep{zhou2021nearly} where $\kappa = K = 1$, the regret is bounded by $\tilde{O}\left(d \sqrt{\sum_{t = 1}^T \sigma_t^2} + (R + 1)\sqrt{dT}\right)$, which matches with the result in \citet{zhou2022computationally} by an $\tilde{O}((R + 1)\sqrt{dT})$ lower-order term. Applying a more fine-grained concentration bound \citep[e.g.,][]{zhou2022computationally} may further remove this term, which we leave for future work. 
\end{remark}

\section{Conclusion and Future Work}

In this paper, we study the problem of stochastic online generalized linear regression and provide a novel analysis for FTRL, attaining an $O\left(\sigma^2 d \log T\right) + o(\log T)$ upper bound. In addition, we prove the first lower bound for online linear regression in the stochastic setting, indicating that our regret bound is minimax-optimal. 

As an application, we further considered heteroscedastic generalized linear bandit problem. Applying parallel FTRL learners, we design a UCB-based algorithm MOR-UCB, which achieves a tighter instance-dependent regret bound in bandit setting. 

Although a near optimal regret for stochastic online linear regression is achieved in this paper, the regret of stochastic online regression of general loss functions is still understudied, which we leave for future work.


\appendix 

\section{Proofs from Section \ref{sec:4}}

\subsection{Proof of Theorem \ref{thm:ftrl}}

\begin{lemma}[Regret decomposition] \label{lemma:decomposition:re}
For each $t \in [T]$, let $\cL_t$ be the cumulative loss function defined in \eqref{eq:def:cumloss} and $\Hb_t$ be the corresponding Hessian matrix as shown in \eqref{eq:def:hessian}. There exists a sequence $\{\bmu'_t\}_{t \in [T]}$ in $\RR^d$ such that the stochastic regret of Algorithm \ref{alg:FTRL} can be decomposed as follows: \begin{align} 
    \cR^{\stoch}(T) \le \lambda B^2 + \frac{1}{2} \sum_{t = 1}^T \left(\phi(\xb_t^\top \hat\bmu_t) - \phi(\xb_t^\top \bmu^*)\right)^2\|\xb_t\|_{\Hb_{t}^{-1}(\bmu_t')}^2 + \frac{1}{2} \sum_{t = 1}^T \epsilon_t^2 \|\xb_t\|_{\Hb_{t}^{-1}(\bmu_t')}^2.  \notag
\end{align}
\end{lemma}

\begin{proof} 
    From the updating rule of Algorithm \ref{alg:FTRL}, 
    \begin{align} 
        \cR^{\stoch}(T) &= \sum_{t = 1}^T \ell_t(\hat\bmu_t) - \sum_{t = 1}^T \ell_t(\bmu^*) \notag
        \\&= \sum_{t = 1}^T \left[\ell_{t}(\hat\bmu_t) + \sum_{\tau = 1}^{t - 1} \ell_\tau(\hat\bmu_t) - \sum_{\tau = 1}^{t} \ell_\tau(\hat\bmu_{t + 1})\right] + \sum_{t = 1}^T \ell_t(\hat \bmu_{t + 1}) - \sum_{t = 1}^T \ell_t(\bmu^*) \notag
        \\&= \sum_{t = 1}^T \left[\cL_t(\hat\bmu_t) - \cL_t(\hat\bmu_{t + 1}) - \phi(\hat\bmu_t) + \phi(\hat\bmu_{t + 1})\right] + \sum_{t = 1}^T \ell_t(\hat \bmu_{t + 1}) - \sum_{t = 1}^T \ell_t(\bmu^*) \notag
        \\&= \sum_{t = 1}^T \left[\cL_t(\hat\bmu_t) - \cL_t(\hat\bmu_{t + 1})\right] + \cL_T(\hat \bmu_{t + 1}) - \cL_T(\bmu^*) + \lambda\|\bmu^*\|_2^2 - \lambda\|\hat\bmu_1\|_2^2 \notag
        \\&\le \lambda\|\bmu^*\|_2^2 + \sum_{t = 1}^T \left[\cL_t(\hat\bmu_t) - \cL_t(\hat\bmu_{t + 1})\right], \label{eq:regret:dec1}
    \end{align}
    where the third equality holds due to the definition of $\cL$ in \eqref{eq:def:cumloss}, 

    Applying Taylor expansion, we have \begin{align} 
        \cL_t(\hat\bmu_t) - \cL_t(\hat\bmu_{t + 1}) &= \left\la \frac{\partial \cL_t}{\partial \bmu}(\hat\bmu_t), \hat\bmu_t - \hat\bmu_{t + 1} \right\ra - \left(\hat\bmu_{t + 1} - \hat \bmu_t\right)^\top \Hb_{t}(\bmu_t')\left(\hat\bmu_{t + 1} - \hat \bmu_t\right)  \notag
        \\&= \left\la \left(\phi(\xb_t^\top \hat\bmu_t) - y_t\right) \xb_t, \hat\bmu_t - \hat\bmu_{t + 1}  \right \ra - \left(\hat\bmu_{t + 1} - \hat \bmu_t\right)^\top \Hb_{t}(\bmu_t')\left(\hat\bmu_{t + 1} - \hat \bmu_t\right) \label{eq:regret:dec2}
    \end{align} for some $\bmu_t' \in \RR^d, \|\bmu_t'\|_2 \le B$. 

    Substituting \eqref{eq:regret:dec2} into \eqref{eq:regret:dec1}, \begin{align} 
        \cR^{\stoch}(T) &\le \lambda B^2 + \sum_{t = 1}^T \left[\left\la \left(\phi(\xb_t^\top \hat\bmu_t) - y_t\right) \xb_t, \hat\bmu_t - \hat\bmu_{t + 1}  \right \ra - \left(\hat\bmu_{t + 1} - \hat \bmu_t\right)^\top \Hb_{t}(\bmu_t')\left(\hat\bmu_{t + 1} - \hat \bmu_t\right)\right] \notag
        \\&\le \lambda B^2 + \frac{1}{4} \sum_{t = 1}^T \left(\phi(\xb_t^\top \hat\bmu_t) - y_t\right)^2 \|\xb_t\|_{\Hb_{t}^{-1}(\bmu_t')}^2 \notag
        \\&\le \lambda B^2 + \frac{1}{2} \sum_{t = 1}^T \left(\phi(\xb_t^\top \hat\bmu_t) - \phi(\xb_t^\top \bmu^*)\right)^2\|\xb_t\|_{\Hb_{t}^{-1}(\bmu_t')}^2 + \frac{1}{2} \sum_{t = 1}^T \epsilon_t^2 \|\xb_t\|_{\Hb_{t}^{-1}(\bmu_t')}^2. \notag
    \end{align}
    This completes the proof.
\end{proof}

\begin{lemma}[Connection between squared estimation error and regret] \label{lemma:connect1:re}
    Consider an arbitrary online learner interactively trained with stochastic data for $T$ rounds as described in Section \ref{section:Preliminaries}. If Condition \ref{cond:noise:1} is true for the noise at all the rounds $t \in [T]$, then the following inequality holds with probability at least $1 - \delta$: \begin{align*} 
        \sum_{t = 1}^T \left[\xb_t^\top (\hat\bmu_t - \bmu^*)\right]^2 \le \frac{4}{\kappa} \cR^{\stoch}(T) + \frac{16}{\kappa^2} \cdot \sigma^2 \log(1 / \delta). 
    \end{align*}
\end{lemma}

\begin{proof} 

    We start by considering the definition of stochastic regret and making use of the property of our aforementioned loss function: 
    \begin{align}
        \cR^{\stoch}(T) &= \sum_{t = 1}^T  \ell_t(\hat\bmu_t) - \sum_{t = 1}^T \ell_t(\bmu^*) \notag
        \\&= -\sum_{t = 1}^T \xb_t^\top (\hat\bmu_t - \bmu^*)y_t + \sum_{t = 1}^T \int_{\xb_t^\top \bmu^*}^{\xb_t^\top \hat\bmu_t} \phi(z) \text{d}z \notag
        \\&\ge -\sum_{t = 1}^T \xb_t^\top (\hat\bmu_t - \bmu^*)y_t + \sum_{t = 1}^T \int_{\xb_t^\top \bmu^*}^{\xb_t^\top \hat\bmu_t} \left(\phi(\xb_t^\top \bmu^*) + \kappa (z - \xb_t^\top \bmu^*)\right) \text{d}z \notag
        \\&= -\sum_{t = 1}^T \xb_t^\top (\hat\bmu_t - \bmu^*)\epsilon_t + \sum_{t = 1}^T \frac{1}{2} \cdot \kappa \left[\xb_t^\top (\hat\bmu_t - \bmu^*)\right]^2, \label{eq:connect1}
    \end{align} where the first equality follows from the definition of regret \eqref{eq:def:regret}, the second equality follows from the definition of loss function \eqref{eq:def:loss}, the inequality holds due to Assumption \ref{cond:activation}. 

    Rearranging \eqref{eq:connect1}, it follows that \begin{align} 
        \sum_{t = 1}^T \left[\xb_t^\top (\hat\bmu_t - \bmu^*)\right]^2 \le \frac{2}{\kappa} \cR^{\stoch}(T) + \frac{2}{\kappa} \sum_{t = 1}^T \epsilon_t \cdot \left[\xb_t^\top (\hat\bmu_t - \bmu^*)\right]. \label{eq:connect2}
    \end{align}
    Since $\hat\bmu_t$ is $(\xb_{1:t}, y_{1:t - 1})$-measurable, we can apply Lemma \ref{lemma:generalized-hoeffding} to show that \begin{align} 
        \sum_{t = 1}^T \epsilon_t \cdot \left[\xb_t^\top (\hat\bmu_t - \bmu^*)\right] \le \sqrt{2 \sigma^2 \log(1 / \delta)\sum_{t = 1}^T \left[\xb_t^\top (\hat\bmu_t - \bmu^*)\right]^2 } \label{eq:connect3}
    \end{align} with probability at least $1 - \delta$. 

    Substituting \eqref{eq:connect3} into \eqref{eq:connect2}, we obtain the following high-probability bound for squared estimation error: \begin{align*} 
        \sum_{t = 1}^T \left[\xb_t^\top (\hat\bmu_t - \bmu^*)\right]^2 &\le \frac{2}{\kappa} \cR^{\stoch}(T) + \frac{2}{\kappa} \sqrt{2 \sigma^2 \log(1 / \delta)\sum_{t = 1}^T \left[\xb_t^\top (\hat\bmu_t - \bmu^*)\right]^2 }
        \\&\le \frac{4}{\kappa} \cR^{\stoch}(T) + \frac{16}{\kappa^2} \cdot \sigma^2 \log(1 / \delta), 
    \end{align*}
    where the last inequality follows from Lemma \ref{lemma:sqrt-trick}. 
\end{proof} 

\begin{lemma} \label{lemma:noise-sum:re}
    Suppose that the sequence of noise $\{\epsilon_t\}_{t\in [T]}$ satisfies Condition \ref{cond:noise:1}. For each $t \in [T]$, let $\underline \Hb_t$ be the matrix defined in \eqref{eq:def:hessian1}. With probability at least $1 - \delta$, \begin{align} 
        \sum_{t = 1}^T \epsilon_t^2 \|\xb_t\|_{\UHb_t^{-1}}^2 \le 34\kappa^{-1} \cdot \sigma^2 d \log \frac{d \lambda + T \kappa A^2}{d \lambda} + 24 \sigma^2 \cdot \frac{A^2}{\lambda} \log(1 / \delta). \notag
    \end{align}
\end{lemma}

\begin{proof} 
    We first prove that the random variable $\epsilon_t^2$ is sub-exponential conditioning on $\xb_{1:t}, y_{1:t - 1}$. 

    Let $v_t = \EE[\epsilon_t^2]$. 
    Considering the moment generating function of $\epsilon_t^2$, we have 
    for all $s \in \RR$, 
    \begin{align} 
        \EE[\exp(s (\epsilon_t^2 - v_t))] &= 1 + s \EE[\epsilon_t^2 - v_t] + \sum_{i = 2}^\infty \frac{s^i}{i!} \EE\left[\left(\epsilon_t^2 - v_t\right)^i\right] \le 1 + \sum_{i = 2}^\infty \frac{s^i}{i!} \EE\left[\epsilon_t^{2i}\right]. \notag
    \end{align}
    For sub-Gaussian noise $\epsilon_t$, we have \begin{align} 
        \EE\left[|\epsilon_t|^r\right] &= \int_{0}^\infty \PP\left(|\epsilon_t|^r \ge x\right) \text{d}x \notag
        \\&= r \int_0^\infty x^{r - 1}\PP\left(|\epsilon_t| \ge x\right) \text{d}x \notag
        \\&\le 2 r \int_0^\infty x^{r - 1} \exp\left(-\frac{x^2}{2\sigma^2}\right) \text{d}x \notag
        \\&= 2^{r / 2} \cdot r \sigma^r \int_{0}^\infty x^{\frac{r}{2} - 1} \exp(-x) \text{d}x \notag
        \\&= 2^{r / 2} \cdot r \sigma^r \cdot  \Gamma(r / 2). \notag
    \end{align}
    Hence, we have \begin{align} 
        \EE[\exp(s (\epsilon_t^2 - v_t))] &\le 1 + \sum_{i =2}^\infty 2 i \cdot \frac{s^i}{i!} 2^i \sigma^{2i} (i - 1)! \notag
        \\&\le 1 + \sum_{i = 2}^\infty 2 (2s\sigma^2)^i \notag
        \\&= 1 + \frac{8s^2\sigma^4}{1 - 2s\sigma^2}, \notag
    \end{align}
    which implies that $\epsilon_t^2 - v_t$ is $\left(\left(4\sqrt{2}\sigma^2\right)^2, 4\sigma^2\right)$-sub-exponential. 

    By the composition property of sub-exponential random variables, we have \begin{align*} \sum_{t = 1}^T \epsilon_t^2 \|\xb_t\|_{\UHb_t^{-1}}^2 \sim \text{SE}\left(32\sigma^4 \sum_{t = 1}^T \|\xb_t\|_{\UHb_t^{-1}}^4, 4\sigma^2 \max_{t \in [T]} \|\xb_t\|_{\UHb_t^{-1}}^2\right). 
    \end{align*}
    By Lemma \ref{lemma:subexponential}, the following concentration bound holds with probability at least $1 - \delta$: 
    \begin{align} 
        \sum_{t = 1}^T \epsilon_t^2 \|\xb_t\|_{\UHb_t^{-1}}^2 &- \sum_{t = 1}^T \EE\left[\epsilon_t^2\right] \|\xb_t\|_{\UHb_t^{-1}}^2 \notag
        \\&\le \max \left\{ 8 \cdot \sqrt{\log (1 / \delta)}\cdot \sigma^2 \sqrt{\sum_{t = 1}^T \|\xb_t\|_{\UHb_t^{-1}}^4},\  8\sigma^2 \max_{t \in [T]} \|\xb_t\|_{\UHb_t^{-1}}^2 \cdot \log(1 / \delta) \right\}. 
    \end{align}
    Applying Lemma \ref{lemma:potential2} and the definition of $\UHb$, we further have \begin{align} 
        \sum_{t = 1}^T \epsilon_t^2 \|\xb_t\|_{\UHb_t^{-1}}^2 &- \sum_{t = 1}^T \EE\left[\epsilon_t^2\right] \|\xb_t\|_{\UHb_t^{-1}}^2 \le 8\sigma^2  \sqrt{\frac{2A^2d}{\lambda} \kappa^{-1} \log \left(\frac{d \lambda + \kappa T A^2}{d \lambda}\right)\log(1 / \delta)} + 8\sigma^2 \frac{A^2}{\lambda} \log(1 / \delta) \notag
    \end{align} with probability at least $1 - \delta$. 

    Since $\epsilon_t$ is $\sigma$-sub-Gaussian, its variance is no larger than $\sigma^2$, which indicates that \begin{align} 
        \sum_{t = 1}^T \epsilon_t^2 \|\xb_t\|_{\UHb_t^{-1}}^2 &\le \left(\sum_{t = 1}^T \epsilon_t^2 \|\xb_t\|_{\UHb_t^{-1}}^2 - \sum_{t = 1}^T \EE\left[\epsilon_t^2\right] \|\xb_t\|_{\UHb_t^{-1}}^2\right) + \sigma^2 \sum_{t = 1}^T \|\xb_t\|_{\UHb_t^{-1}}^2 \notag
        \\&\le 8\sigma^2 \cdot \kappa^{-1 / 2}\sqrt{\frac{2A^2d}{\lambda} \log \left(\frac{d \lambda + T A^2}{d \lambda}\right)\log(1 / \delta)} + 8\sigma^2 \cdot \frac{A^2}{\lambda} \log(1 / \delta) \notag \\&\quad + \sigma^2 \cdot 2d \kappa^{-1} \log \frac{d \lambda + T \kappa A^2}{d \lambda} \notag
        \\&\le 34\kappa^{-1} \cdot \sigma^2 d \log \frac{d \lambda + T \kappa A^2}{d \lambda} + 24 \sigma^2 \cdot \frac{A^2}{\lambda} \log(1 / \delta) \notag
    \end{align}
    with probability at least $1 - \delta$. 

\end{proof}

\begin{proof}[Proof of Theorem \ref{thm:ftrl}]
    Based on Lemma \ref{lemma:connect1} and Lemma \ref{lemma:noise-sum}, the following two inequalities hold simultaneously with probability at least $1 - 2\delta$ for all $\delta \in (0, \frac{1}{2})$: \begin{align} 
        \sum_{t = 1}^T \left[\xb_t^\top (\hat\bmu_t - \bmu^*)\right]^2 &\le \frac{4}{\kappa} \cR^{\stoch}(T) + \frac{16}{\kappa^2} \cdot \sigma^2 \log(1 / \delta),   \label{eq:connect}\\
        \sum_{t = 1}^T \epsilon_t^2 \|\xb_t\|_{\UHb_t^{-1}}^2 &\le 34\kappa^{-1} \cdot \sigma^2 d \log \frac{d \lambda + T \kappa A^2}{d \lambda} + 24 \sigma^2 \cdot \frac{A^2}{\lambda} \log(1 / \delta).  \label{eq:noise-sum}
    \end{align}
In the remaining proof, we assume that \eqref{eq:connect} and \eqref{eq:noise-sum} hold. 

    From Lemma \ref{lemma:decomposition}, we have \begin{align} 
        \cR^{\stoch}(T) &\le \lambda B^2 + \frac{1}{2} \sum_{t = 1}^T \left(\phi(\xb_t^\top \hat\bmu_t) - \phi(\xb_t^\top \bmu^*)\right)^2\|\xb_t\|_{\Hb_{t}^{-1}(\bmu_t')}^2 + \frac{1}{2} \sum_{t = 1}^T \epsilon_t^2 \|\xb_t\|_{\Hb_{t}^{-1}(\bmu_t')}^2 \notag
        \\&\le \lambda B^2 + \frac{A^2K^2}{2\lambda} \sum_{t = 1}^T \left[\xb_t^\top (\hat\bmu_t - \bmu^*)\right]^2 + \frac{1}{2} \sum_{t = 1}^T \epsilon_t^2 \|\xb_t\|_{\Hb_{t}^{-1}(\bmu_t')}^2 \notag
        \\&\le \lambda B^2 + \frac{2A^2K^2}{\lambda \kappa} \cR^{\stoch}(T) + \frac{8A^2K^2}{\lambda \kappa^2} \sigma^2 \log(1 / \delta) \notag \\&\quad+ 17\kappa^{-1} \cdot \sigma^2 d \log \frac{d \lambda + T \kappa A^2}{d \lambda} +  12 \sigma^2 \cdot \frac{A^2}{\lambda} \log(1 / \delta) \label{eq:regret-lambda}
    \end{align} 
    where the first inequality is given by Lemma \ref{lemma:connect1} directly, the second inequality follows from the Lipschitz property of the activation function in Assumption \ref{cond:activation}, the third inequality holds due to \eqref{eq:connect}, \eqref{eq:noise-sum} and the fact that $\UHb_t \preceq \Hb_t^{-1}(\bmu_t')$. 

    Substituting $\lambda = 4A^2K^2 / \kappa$ into \eqref{eq:regret-lambda}, we have \begin{align*}
        \cR^{\stoch}(T) &\le \frac{1}{2} \cR^{\stoch}(T) + \left(\frac{2}{\kappa} + \frac{3\kappa}{K^2}\right) \sigma^2 \log(1 / \delta) + 17 \kappa^{-1} \cdot \sigma^2 d \log \frac{4dK^2 + T \kappa^2}{4dK^2} + 4\frac{K^2A^2B^2}{\kappa}
        \\&\le 34 \kappa^{-1} \cdot \sigma^2 d \log \frac{4dK^2 + T \kappa^2}{4dK^2} + 8\frac{K^2A^2B^2}{\kappa} + 2\left(\frac{2}{\kappa} + \frac{3\kappa}{K^2}\right) \sigma^2 \log(1 / \delta), 
    \end{align*}
    which completes the proof. 
\end{proof}

\begin{theorem}[Theorem 3.1, \citealt{maillard2021stochastic}]
In stochastic online linear regression $(\kappa = K = 1)$ in Assumption \ref{cond:activation} with Condition \ref{cond:noise:1}, we have with probability at least $1 - \delta$, \begin{align*} \cR^\adv(T) - \cR^\stoch(T) \le O(\sigma^2 d \log T) + o(\log T). \end{align*}
\end{theorem}

\subsection{Proof of Theorem \ref{thm:ftrl2}} \label{section:proof:2}

\begin{lemma}[Connection between squared estimation error and regret] \label{lemma:connect2}
    Consider an arbitrary online learner interactively trained with stochastic data for $T$ rounds as described in Section \ref{section:Preliminaries}. If Condition \ref{cond:noise:2} is true for the noise at all the rounds $t \in [T]$, then the following inequality holds with probability at least $1 - \delta$: \begin{align*} 
        \sum_{t = 1}^T \left[\xb_t^\top (\hat\bmu_t - \bmu^*)\right]^2 \le \frac{4}{\kappa} \cR^{\stoch}(T) + \frac{16}{\kappa^2} \cdot R^2 \log(1 / \delta). 
    \end{align*}
\end{lemma}

\begin{lemma} \label{lemma:noise-sum-2}
    Suppose that the sequence of noise $\{\epsilon_t\}_{t\in [T]}$ satisfies Condition \ref{cond:noise:2}. For each $t \in [T]$, let $\underline \Hb_t$ be the matrix defined in \eqref{eq:def:hessian1}. With probability at least $1 - \delta$, \begin{align} 
        \sum_{t = 1}^T \epsilon_t^2 \|\xb_t\|_{\UHb_t^{-1}}^2 \le 6\kappa^{-1} \cdot\sigma^2  d \log \frac{d \lambda + T \kappa A^2}{d \lambda} + \frac{5}{3} \cdot \frac{R^2A^2}{\lambda} \log(1 / \delta).  \notag
    \end{align}
\end{lemma}

\begin{proof} 
    We prove this lemma by bounding $\sum_{t = 1}^T \EE[\epsilon_t^2] \|\xb_t\|_{\UHb_{t}^{-1}}^2$ and $\sum_{t = 1}^T \left(\epsilon_t^2 - \EE[\epsilon_t^2]\right)\|\xb_t\|_{\UHb_{t}^{-1}}^2$ separately. 

    For the first term, we have \begin{align} 
        \sum_{t = 1}^T \EE[\epsilon_t^2] \|\xb_t\|_{\UHb_{t}^{-1}}^2 &\le \sum_{t = 1}^T \sigma^2 \|\xb_t\|_{\UHb_{t}^{-1}}^2 \le 2 \kappa^{-1} \cdot \sigma^2 d \log \frac{d \lambda + T \kappa A^2}{d \lambda}. \label{eq:noise-sum-2:mean}
    \end{align}
    For the second term $\sum_{t = 1}^T \left(\epsilon_t^2 - \EE[\epsilon_t^2]\right)\|\xb_t\|_{\UHb_{t}^{-1}}^2$, it holds that \begin{align} 
        \EE\left[\left(\epsilon_t^2 - \EE[\epsilon_t^2]\right)\|\xb_t\|_{\UHb_{t}^{-1}}^2\right] &= 0, \notag \\
        \sum_{t = 1}^T \var\left[\left(\epsilon_t^2 - \EE[\epsilon_t^2]\right)\|\xb_t\|_{\UHb_{t}^{-1}}^2\right] &\le \sum_{t = 1}^T \EE[\epsilon_t^4] \|\xb_t\|_{\UHb_{t}^{-1}}^4 \notag\\&\le R^2 \sigma^2  \frac{A^2}{\lambda} \sum_{t = 1}^T \|\xb_t\|_{\UHb_t^{-1}}^2 \notag \\&\le 2\kappa^{-1} \lambda^{-1} \cdot R^2 \sigma^2 {A^2} d \log \frac{d \lambda + T \kappa A^2}{d \lambda}, \notag \\
        \left|\left(\epsilon_t^2 - \EE[\epsilon_t^2]\right) \|\xb_t\|_{\UHb_t^{-1}}^2\right| &\le R^2 \cdot \frac{A^2}{\lambda}.  \notag
    \end{align}
    Applying Lemma \ref{lemma:freedman}, with probability at least $1 - \delta$, \begin{align}
        \sum_{t = 1}^T \left(\epsilon_t^2 - \EE[\epsilon_t^2]\right) \|\epsilon_t\|_{\UHb_t^{-1}}^2 &\le 2R \sigma A\sqrt{\kappa^{-1}\lambda^{-1} d \log \frac{d \lambda + T \kappa A^2}{d \lambda} \log(1  / \delta)} + \frac{2}{3} \cdot \frac{R^2A^2}{\lambda} \log(1 / \delta) \notag
        \\&\le 4\sigma^2 \kappa^{-1} \cdot d \log \frac{d \lambda + T \kappa A^2}{d \lambda} + \frac{5}{3} \cdot \frac{R^2A^2}{\lambda} \log(1 / \delta). \label{eq:noise-sum-2:var}
    \end{align}
    Combining \eqref{eq:noise-sum-2:mean} with \eqref{eq:noise-sum-2:var}, we can show that with probability at least $1 - \delta$, \begin{align*}
        \sum_{t = 1}^T \epsilon_t^2 \|\xb_t\|_{\UHb_t^{-1}}^2 \le 6\kappa^{-1} \cdot\sigma^2  d \log \frac{d \lambda + T \kappa A^2}{d \lambda} + \frac{5}{3} \cdot \frac{R^2A^2}{\lambda} \log(1 / \delta). 
    \end{align*}
\end{proof}

\begin{proof}[Proof of Theorem \ref{thm:ftrl2}]
    Based on Lemma \ref{lemma:connect2} and Lemma \ref{lemma:noise-sum-2}, the following two inequalities hold simultaneously with probability at least $1 - 2\delta$ for all $\delta \in (0, \frac{1}{2})$: \begin{align} 
        \sum_{t = 1}^T \left[\xb_t^\top (\hat\bmu_t - \bmu^*)\right]^2 &\le \frac{4}{\kappa} \cR^{\stoch}(T) + \frac{16}{\kappa^2} \cdot R^2 \log(1 / \delta),   \label{eq:connect-2}\\
        \sum_{t = 1}^T \epsilon_t^2 \|\xb_t\|_{\UHb_t^{-1}}^2 &\le 6\kappa^{-1} \cdot\sigma^2  d \log \frac{d \lambda + T \kappa A^2}{d \lambda} + \frac{5}{3} \cdot \frac{R^2A^2}{\lambda} \log(1 / \delta).  \label{eq:noise-sum-2}
    \end{align}
    In the remaining proof, we assume that \eqref{eq:connect-2} and \eqref{eq:noise-sum-2} hold. 

    From Lemma \ref{lemma:decomposition}, we have \begin{align} 
        \cR^{\stoch}(T) &\le \lambda B^2 + \frac{1}{2} \sum_{t = 1}^T \left(\phi(\xb_t^\top \hat\bmu_t) - \phi(\xb_t^\top \bmu^*)\right)^2\|\xb_t\|_{\Hb_{t}^{-1}(\bmu_t')}^2 + \frac{1}{2} \sum_{t = 1}^T \epsilon_t^2 \|\xb_t\|_{\Hb_{t}^{-1}(\bmu_t')}^2 \notag
        \\&\le \lambda B^2 + \frac{A^2K^2}{2\lambda} \sum_{t = 1}^T \left[\xb_t^\top (\hat\bmu_t - \bmu^*)\right]^2 + \frac{1}{2} \sum_{t = 1}^T \epsilon_t^2 \|\xb_t\|_{\Hb_{t}^{-1}(\bmu_t')}^2 \notag
        \\&\le \lambda B^2 + \frac{2A^2K^2}{\lambda \kappa} \cR^{\stoch}(T) + \frac{8A^2K^2}{\lambda \kappa^2} R^2 \log(1 / \delta) \notag \\&\quad+ 3\kappa^{-1} \cdot\sigma^2  d \log \frac{d \lambda + T \kappa A^2}{d \lambda} + \frac{5}{6} \cdot \frac{R^2A^2}{\lambda} \log(1 / \delta) \label{eq:regret-lambda-2}
    \end{align} 
    where the first inequality is given by Lemma \ref{lemma:connect1} directly, the second inequality follows from the Lipschitz property of the activation function in Assumption \ref{cond:activation}, the third inequality holds due to \eqref{eq:connect-2}, \eqref{eq:noise-sum} and the fact that $\UHb_t \preceq \Hb_t^{-1}(\bmu_t')$. 

    Substituting $\lambda = 4A^2K^2 / \kappa$ into \eqref{eq:regret-lambda}, we have \begin{align*}
        \cR^{\stoch}(T) &\le \frac{1}{2} \cR^{\stoch}(T) + \left(\frac{2}{\kappa} + \frac{5\kappa}{24K^2}\right) R^2 \log(1 / \delta) + 3 \kappa^{-1} \cdot \sigma^2 d \log \frac{4dK^2 + T \kappa^2}{4dK^2} + 4\frac{K^2A^2B^2}{\kappa}
        \\&\le 6 \kappa^{-1} \cdot \sigma^2 d \log \frac{4dK^2 + T \kappa^2}{4dK^2} + 8\frac{K^2A^2B^2}{\kappa} + 2\left(\frac{2}{\kappa} + \frac{5\kappa}{24K^2}\right) R^2 \log(1 / \delta), 
    \end{align*}
    which completes the proof. 
\end{proof}

\subsection{Proof of An Additional Result}

In this subsection, we present the following theorem, which provides a high-probability upper bound for the `gap' between the online ridge regression estimator and $\bmu^*$. 

\begin{theorem}[Confidence ellipsoid for ridge regression estimator]  \label{thm:ellipsoid}
    Set $\lambda = 4A^2 K^2 / \kappa$ and assume that the noise $\epsilon_t$ satisfy Condition \ref{cond:noise:2} at all rounds $t \in [T]$, then with probability at least $1 - \delta$, for all $t \in [T]$, it holds that \begin{align} 
        \|\hat \bmu_t - \bmu^* \|_{\UHb_{t}} &\le 8\sigma \sqrt{d \log \left(\frac{2d\lambda + t\kappa A^2}{2d\lambda}\right) \log \left(4t^2 / \delta\right)} \notag  + 4R \log(4t^2 / \delta) + \sqrt{2\lambda} B. \notag
    \end{align}
\end{theorem}

\begin{remark} 
    This theorem elucidates how to construct a confidence ellipsoid with predictions given by FTRL. Similar variance-aware confidence sets have been shown by \citet{zhou2021nearly, zhang2021variance} in linear regression, while Theorem \ref{thm:ellipsoid} is applicable to generalized linear function class. Later in section \ref{section:bandits}, we will show how to make use of this theorem in bandit setting. 
\end{remark}

\begin{proof} 
    According to Algorithm \ref{alg:FTRL}, $\bmu_{t + 1}$ is the minimizer of $\lambda \|\bmu\|_2^2 + \sum_{\tau = 1}^{t} \ell_\tau(\bmu)$. 

    Taking the derivative, we have \begin{align}
        0 &= 2\lambda \hat \bmu_{t + 1} + \sum_{\tau \in [t]} \left(-y_\tau\cdot \xb_\tau + \phi(\xb_\tau^\top \hat \bmu_{t + 1}) \cdot \xb_\tau \right) \notag \\&= 
        2\lambda \hat \bmu_{t + 1} + \sum_{\tau \in[t]} \left( -y_{\tau} + \phi(\xb_\tau^\top \bmu^*)\right) \cdot \xb_{\tau} + \sum_{\tau \in [t]} \left(\phi(\xb_\tau^\top \hat \bmu_{t + 1}) - \phi(\xb_\tau^\top \bmu^*)\right) \cdot \xb_\tau \notag 
    \end{align}
    Rearranging the equality, \begin{align}
        \sum_{\tau \in [t]} \left(\phi(\xb_\tau^\top \hat \bmu_{t + 1}) - \phi(\xb_\tau^\top \bmu^*)\right) \cdot \xb_\tau  + 2\lambda \left(\hat \bmu_{t + 1} - \bmu^*\right) = \sum_{\tau \in [t]} \epsilon_\tau \xb_\tau - 2\lambda \cdot \bmu^* \notag
    \end{align}
    For short, we let $\kappa_{\tau, t + 1} = \frac{\phi(\xb_\tau^\top \hat \bmu_{t + 1})- \phi(\xb_\tau^\top \bmu^*)}{\xb_\tau^\top \hat \bmu_{t + 1} - \xb_\tau^\top \bmu^*}$. By Assumption \ref{cond:activation}, $\kappa_{\tau, t + 1} \in [\kappa, K]$. 

    Thus, we have \begin{align} 
        \left\|\left(2\lambda \cdot \Ib + \sum_{\tau \in [t]} \kappa_{\tau, t + 1} \xb_\tau \xb_\tau^\top\right) \cdot (\hat\bmu_{t + 1} - \bmu^*)\right\|_{\UHb_{t}^{-1}} &= \left\|\sum_{\tau \in [t]} \epsilon_\tau \xb_\tau - 2\lambda \cdot \bmu^*\right\|_{\UHb_{t}^{-1}} \notag
        \\&\le \left\|\sum_{\tau \in [t]} \epsilon_\tau \xb_\tau\right\|_{\UHb_{t}^{-1}} + \sqrt{2\lambda} \cdot \left\|\bmu^*\right\|_2. \notag
    \end{align}
    Since $2\lambda \cdot \Ib + \sum_{\tau \in [t]} \kappa_{\tau, t + 1} \xb_\tau \xb_\tau^\top \succeq \UHb_t$, with probability at least $1 - \delta$, for all $t \ge 1$, it holds that
    \begin{align} 
        \left\|\hat\bmu_{t + 1} - \bmu^* \right\|_{\UHb_t^{-1}} &\le \left\|\sum_{\tau \in [t]}\epsilon_\tau \xb_\tau \right\|_{\UHb_t^{-1}} + \sqrt{2\lambda} \cdot \left\|\bmu^*\right\|_2 \notag
        \\&\le 8\sigma \cdot \kappa^{-1 / 2} \sqrt{d \log \left(\frac{2d\lambda + t\kappa A^2}{2d\lambda}\right) \log \left(4t^2 / \delta\right)} + 4\cdot \kappa^{-1 / 2}R \log(4t^2 / \delta) + \sqrt{2\lambda} B,  \notag
    \end{align}
    where the second inequality holds due to Theorem 4.1 in \citet{zhou2021nearly}. 
\end{proof}

\section{Proofs from Section \ref{section:bandits}}

\begin{lemma}[Variance-aware confidence ellipsoid for generalized linear bandits] \label{lemma:conf}
    Suppose that $\|\btheta^*\|_2 \le 1$ and for all $\ab \in \bigcup_{t \in [T]} \cD_t$, $\|\ab\|_2 \le 1$. 
    Set $\lambda = 4K^2 / \kappa$ in Algorithm \ref{alg:1}. With probability at least $1 - \delta$, it holds that \begin{align} 
        \|\hat\btheta_{t, \ell} - \btheta^*\|_{\bSigma_{t, l}} &\le 16\cdot 2^{l} \overline \sigma \cdot \kappa^{-1 / 2} \sqrt{d \log \left(\frac{2d\lambda + t\kappa A^2}{2d\lambda}\right) \log \left(4t^2 L/ \delta\right)} \\& + 4R \cdot \kappa^{-1 / 2} \log(4t^2 L/ \delta) + 2\sqrt{2 / \kappa}\cdot K \notag
    \end{align} for all $t \in [T]$. 
\end{lemma}

\begin{proof} 
This lemma can be proved by a direct application of Theorem \ref{thm:ellipsoid} and a union bound over $L$ layers. 
\end{proof}

\begin{theorem}[Cumulative regret for generalized linear bandits]
    Suppose that $\|\btheta^*\|_2 \le 1$ and for all $\ab \in \bigcup_{t \in [T]} \cD_t$, $\|\ab\|_2 \le 1$. 
    Set $\lambda = 4K^2 / \kappa$ and \begin{align} 
        \beta_{t, l} = 16\cdot 2^{l} \overline \sigma \cdot \kappa^{-1 / 2} \sqrt{d \log \left(\frac{2d\lambda + t\kappa A^2}{2d\lambda}\right) \log \left(4t^2 L/ \delta\right)} + 4R \cdot \kappa^{-1 / 2} \log(4t^2 L/ \delta) + 2\sqrt{2 / \kappa}\cdot K \label{eq:beta}
    \end{align} in Algorithm \ref{alg:1}. With probability at least $1 - \delta$, the regret of Algorithm \ref{alg:1} at the first $T$ rounds satisfies that: \begin{align*} 
        \regret(T) \le \tilde{O}\left(\frac{K}{\kappa} \cdot d \sqrt{\sum_{t = 1}^T\sigma_t^2} + \left(R + K\right) K \cdot \kappa^{-1} \sqrt{dT}\right)
    \end{align*}
\end{theorem}

\begin{proof} 
    Based on the definition of regret \eqref{def:regret}, 
    \begin{align} 
        \regret(T) &= \sum_{t = 1}^T \phi\left(\left(\btheta^*\right)^\top \ab_t^*\right) - \phi\left(\left(\btheta^*\right)^\top \ab_t\right) \notag
        \\&\le \sum_{t = 1}^T \phi\left({\hat \btheta_{t, l_t}}^\top \ab_t + \beta_{t, l_t} \cdot \|\ab_t\|_{\bSigma_{t, l_t}^{-1}}\right) - \phi\left(\left(\btheta^*\right)^\top \ab_t\right) \notag
        \\&\le \sum_{t = 1}^T 2K \cdot \beta_{t, l_t} \cdot \|\ab_t\|_{\bSigma_{t, l_t}^{-1}} \notag
        \\&\le 2K \sum_{l \in [L]} \beta_{T, l} \sum_{t \in \Psi_{T + 1, l}} \|\ab_t\|_{\bSigma_{t, l}^{-1}} \notag
        \\&\le 2K \sum_{l \in [L]} \beta_{T, l} \sqrt{\left|\Psi_{T+1, l}\right|}\sqrt{\sum_{t \in \Psi_{T + 1, l}} \min\left\{1 / \kappa, \|\ab_t\|_{\bSigma_{t, l}^{-1}}^2\right\}} \notag
        \\&\le 4K \sum_{l \in [L]} \beta_{T, l} \sqrt{\left|\Psi_{T+1, l}\right|} \cdot \sqrt{d \cdot \kappa^{-1} \log \left(\frac{2d\lambda + T\kappa A^2}{2d\lambda}\right)}, 
        \label{eq:thm-regret1}
    \end{align}
    where the first inequality holds due to Lemma \ref{lemma:conf}, the second inequality follows from Assumption \ref{cond:activation}, the fourth inequality is obtained by applying Cauchy-Schwarz inequality, the last inequality follows from Lemma \ref{lemma:potential}. 

    Substituting \eqref{eq:beta} into \eqref{eq:thm-regret1}, we obtain \begin{align}
        \regret(T) &\le 4K \sqrt{d \cdot \kappa^{-1} \log \left(\frac{2d\lambda + T\kappa A^2}{2d\lambda}\right)} \sum_{l \in [L]} \sqrt{\left|\Psi_{T+1, l}\right|} \cdot \tilde{O}\left(2^l \overline{\sigma} \kappa^{-1 / 2} \sqrt{d} + R \cdot \kappa^{-1 / 2} + \sqrt{K^2 /\kappa}\right) \notag
        \\&\le 4K \sqrt{L} \sqrt{d \cdot \kappa^{-1} \log \left(\frac{2d\lambda + T\kappa A^2}{2d\lambda}\right)} \sqrt{\sum_{l \in [L]}\sum_{t \in \Psi_{T + 1, l}} \tilde{O}\left( \sigma_t^2\cdot \kappa^{-1} d + R^2 / \kappa + K^2 /\kappa\right)} \notag
        \\&\le \tilde{O}\left(\frac{K}{\kappa} \cdot d \sqrt{\sum_{t = 1}^T\sigma_t^2} + \left(R + K\right) K \cdot \kappa^{-1} \sqrt{dT}\right)
    \end{align}
\end{proof}

\input{Auxiliary.tex}

\bibliographystyle{ims}
\bibliography{reference}


\end{document}

%% file: Preliminaries.tex
\section{Preliminaries} \label{section:Preliminaries}

We will introduce our problem setting and some basic concepts in this section. 

\subsection{Problem Setup}\label{nnnn}

\noindent \textbf{Stochastic online regression.}  Let $T$ be the number of rounds. At each round $t \in [T]$, the learner observes a feature vector $\xb_t \in \RR^d$ which is arbitrarily generated by the environment with $\|\xb_t\|_2 \le A$. The environment also generates the true label based on underlying function $f_{\bmu^*}$ parameterized by $\bmu^*$ with $\|\bmu^*\|_2 \le B$ and a stochastic noise $\epsilon_t$, i.e. true label $y_t := f_{\bmu^*}(\xb_t) + \epsilon_t$. After observing $\xb_t$, the learner should output a prediction $\hat y_t = f_{\hat \bmu_t}(\xb_t) \in \RR$ where $\hat \bmu_t$ stands for its estimation of $\bmu^*$. $y_t$ is subsequently revealed to the learner at the end of the $t$-th round. 


\noindent \textbf{Generalized linear function class.} In this work, we assume that $f_{\bmu^*}$ belongs to the generalized linear function class $\cG$ with known activation function $\phi: \RR \to \RR$ such that \begin{align} \cG = \left\{f_{\bmu} | \bmu \in \RR^d, f_{\bmu}(\xb) = \phi\left(\la \bmu, \xb \ra\right) \text{for } \forall \xb \in \RR^d\right\}  \label{eq:def:glm} \end{align}
To make the regression problem tractable, we require the following assumption on activation function $\phi$, which is a common assumption in literature \citep{filippi2010parametric}
\begin{assumption} \label{cond:activation}
    The activation function $\phi(\cdot)$ is an increasing differentiable function on $[-B, B]$ and there exists $\kappa, K \in \RR$ such that  $0 < \kappa \le \phi'(z) \le K$ for all $z \in [-B, B]$. 
\end{assumption}

\noindent \textbf{Assunmptions on noise.} Two types of noise are considered in this work: 
\begin{condition}[sub-Gaussian noise] \label{cond:noise:1}
    Suppose the noise sequence $\{\epsilon_t\}_{t \in [T]}$ is a sequence of i.i.d. zero-mean sub-Gaussian random variables: 
    \begin{align*} 
        \forall t \ge 1, s \in \RR,\quad \EE[\exp(s \epsilon_t)] \le \exp\left(\frac{\sigma^2s^2}{2}\right). 
    \end{align*}
\end{condition}

\begin{condition}[Bernstein's condition] \label{cond:noise:2}
    The noise sequence $\{\epsilon_t\}_{t \in [T]}$ is a sequence of independent zero-mean random variables such that \begin{align*} 
        \forall t \ge 1, \quad \PP\left(|\epsilon_t| \le R\right) = 1, \EE[\epsilon_t^2] \le \sigma^2. 
    \end{align*}
\end{condition}
\begin{remark}
    Noise with Bernstein condition naturally implies sub-Gaussianity with a variance parameter $R$. However, we want to emphasize that we are more interested in the separate dependence of $R$ and $\sigma$ defined in Condition \ref{cond:noise:2} in the complexity bound we will derive. Simplely regarding Bernstein condition noise as a sub-Gaussian noise will omit the refined dependence we want to have.
\end{remark}

\noindent \textbf{Loss and regret.} Following \citet{jun2017scalable}, 
we define the loss function $\ell_t$ at each round $t \in [T]$ as follows . \begin{align} 
    \ell_t(\bmu) := -\xb_t^\top \bmu y_t + \int_0^{\xb_t^\top \bmu} \phi(z) \text{d}z. \label{eq:def:loss}
\end{align} For linear case where $\phi$ is the identical mapping, the loss function $\ell_t$ is equivalent to the square loss between $y_t$ and $\hat y_t$, i.e. they only differ by a constant: \begin{align*}
    \ell_t(\bmu) = -\xb_t^\top \bmu y_t + \int_0^{\xb_t^\top \bmu} z \text{d}z = \frac{1}{2} \left(\xb_t^\top \bmu - y_t\right)^2 - \frac{1}{2} y_t^2. 
\end{align*} 

\begin{remark} 
    Intuitively, our definition of loss functions $\{\ell_t\}_{t = 1}^T$ is a sequence of negative log likelihood when the distribution of $y$ belongs to the exponential family \citep{nelder1972generalized}, i.e. \begin{align*} 
        \PP(y_t | \xb_t^\top \bmu = z) = h(y, z) \exp\left(\frac{y \cdot z - a(z)}{b}\right),
    \end{align*} where $a'(z) = \phi(z)$. 
\end{remark}

Aligned with the definition proposed by \cite{maillard2021stochastic}, the stochastic regret is defined as the relative cumulative loss over $f_{\bmu^*}$: \begin{align} 
    \cR^{\stoch}(T) := \sum_{t \in [T]} \ell_t(\hat \bmu_t) - \sum_{t \in [T]} \ell_t(\bmu^*). \label{eq:def:regret}
\end{align}

\begin{remark} 
    Consistent with our stochastic setting, the stochastic regret is defined in a more natural way, representing the gap of cumulative loss between the learner and the underlying function $f_{\bmu^*}$. In comparison, the adversarial regret is defined as \begin{align*} 
        \cR^{\adv}(T) := \sum_{t \in [T]} \ell_t(\hat \bmu^t) - \inf_{\bmu \in \RR^d} \sum_{t \in [T]} \ell_t(\bmu).
    \end{align*}
    These two definitions of regret are highly similar as what have discussed in previous work \citep[Theorem 3.1]{maillard2021stochastic}. The similarity has been well-studied for the adversarial bandit and stochastic bandit setting \citep{lattimore2020bandit}. Actually, the two regrets are closed to each other in stochastic online linear regression. As shown later in Section \ref{sec:4}, our bound for stochastic regret also leads to an upper bound for $\cR^\adv$ of the same order. 
    
    
\end{remark}

\subsection{Comparison between Adversarial Setting and Stochastic Setting}

In this subsection, we discuss existing results on the adversarial regret bounds for online linear regression. The minimax regret in the adversarial setting is first derived by \citet{bartlett2015minimax}, in the case where all the feature vectors are known to the learner at the beginning of the first round. The protocol of adversarial online regression is summarized in Figure \ref{figure:1}.  
\begin{figure} 
\fbox{\begin{minipage}{\linewidth - 4mm}
    Given number of rounds $T$, vector sequence $\{\xb_t\}_{t\in T}$. 

    For $t = 1, 2, \cdots, T$:
    \begin{itemize}
        \item \label{model:step:1}At the beginning of round $t$, learner outputs a prediction $\hat y_t$. 
        \item \label{model:step:2} The adversary reveals $y_t \in [-Y, Y] \subset \cR$ to the learner. 
        \item The learner observes $y_t$ and incurs loss $(y_t - \hat y_t)^2$. 
    \end{itemize}
\end{minipage}}
\caption{Protocol of adversarial online linear regression \label{figure:1}}
\end{figure}

Exploiting the adversarial nature of the sequential data, \citet{bartlett2015minimax} directly solved the following minimax regret \begin{align*} 
    \min_{\hat y_1} \max_{y_1} \cdots \min_{\hat y_T} \max_{y_T} \cR^{\adv}(T).
\end{align*}
The optimal minimax regret is $O(Y^2 d \log T)$, while the optimal strategy for the adversary is to set the label $y_t$ according to the following distribution: \begin{align} 
    y_t = \begin{cases}
        Y & w.p.\ \frac{1}{2} + \xb_t \Pb_t \left(\sum_{\tau = 1}^{t - 1} y_\tau \xb_\tau\right) / (2B)  \\
        -Y & w.p.\ \frac{1}{2} - \xb_t \Pb_t \left(\sum_{\tau = 1}^{t - 1} y_\tau \xb_\tau\right) / (2B)
    \end{cases} \label{eq:adv:noise}
\end{align} where $\Pb_t$ is defined by \begin{align*}\Pb_t^{-1} = \sum_{\tau = 1}^t \xb_\tau \xb_\tau^\top + \sum_{\tau = t + 1}^T \frac{\xb_\tau^\top \Pb_\tau \xb_\tau}{1 + \xb_\tau^\top \Pb_\tau \xb_\tau} \xb_\tau \xb_\tau^\top.\end{align*} 
It is not hard to see that such a regret bound still holds in the stochastic setting if we set $Y$ to be sufficiently large since $\inf_{\bmu \in \RR^d} \sum_{t = 1}^T (\la \xb_t, \bmu \ra - y_t)^2 \le \sum_{t = 1}^T (\la \xb_t, \bmu^* \ra - y_t)^2$. 

A natural question that arises here is whether it is significant to have a new regret bound for the stochastic setting. To answer this question, we carry out experiments for \emph{follow-the-regularized-leader} (FTRL) with both bounded stochastic label and bounded adversarial label defined in \eqref{eq:adv:noise}. We conduct this experiment under different types of noise, including clipped Gaussian noise with standard variance 1, 2, 4 and adversarial noise according to \citet{bartlett2015minimax}. The stochastic noise is clipped to ensure that all the labels are in the same bounded interval $[-Y, Y]$. In this experiment, we choose $Y=5$. We plot the regret of FTRL under different types of labels in Figure~\ref{figure:2}. We can see that the regret of FTRL in the stochastic setting is remarkably smaller than that in the adversarial setting, thus the regret bound for the adversarial setting is a crude upper bound of the  regret in the stochastic setting. Besides, our experimental results indicate that the regret of FTRL highly depends on the noise variance, even though the ranges of labels are the same.


%% file: Auxiliary.tex
\section{Auxiliary Lemmas}

\begin{lemma}[\citealt{freedman1975tail}] \label{lemma:freedman}
    Let $M, v> 0$ be fixed constants. Let $\{x_i\}_{i = 1}^n$ be a stochastic process, $\{\cG_i\}_i$ be a filtration so that for all $i \in [n]$, $x_i$ is $\cG_i$-measurable, while most surely $\EE[x_i|\cG_{i - 1}] = 0$, $|x_i| \le M$ and 
        $\sum_{i = 1}^n \EE(x_i^2 | \cG_i) \le v. $
    Then, for any $\delta > 0$, with probability $1 - \delta$, \begin{align} 
        \sum_{i = 1}^n x_i \le \sqrt{2v \log(1 / \delta)} + 2 / 3\cdot M \log (1 / \delta). \notag
    \end{align}
\end{lemma}

\begin{lemma} \label{lemma:sqrt-trick}
    Suppose $a, b \ge 0$. If $x^2 \le a + b \cdot x$, then $x^2 \le 2b^2 + 2a$. \notag
\end{lemma}

\begin{proof}
    By solving the root of quadratic polynomial $q(x):= x^2 - b\cdot x - a$, we obtain $\max\{x_1, x_2\} = (b + \sqrt{b^2 + 4a})/2$. Hence, we have $x \le (b + \sqrt{b^2 + 4a})/2$ provided that $q(x) \le 0$. Then we further have \begin{align} 
        x^2 \le \frac{1}{4} \left(b + \sqrt{b^2 + 4a}\right)^2 \le \frac{1}{4} \cdot 2 \left(b^2 + b^2 + 4a\right) \le 2b^2 + 2a. 
    \end{align}
\end{proof}

\begin{lemma}[Hoeffding's inequality] \label{lemma:generalized-hoeffding}
    Let $\{x_i\}_{i = 1}^n$ be a stochastic process, $\{\cG_i\}_i$ be a filtration so that for all $i \in [n]$, $x_i$ is $\cG_i$-measurable, while $\EE[x_i|\cG_{i - 1}] = 0$ and $x_i|\cG_{i - 1}$ is a $\sigma_i$-sub-Gaussian random variable. Then, for any $t > 0$, with probability at least $1 - \delta$, it holds that \begin{align*} 
        \sum_{i = 1}^n x_i \le \sqrt{2\sum_{i = 1}^n \sigma_i^2\log(1 / \delta)}. 
    \end{align*}
\end{lemma}

\begin{lemma}[Lemma 11, \citealt{abbasi2011improved}] \label{lemma:potential}
    For any $\lambda > 0$ and sequence $\{\xb_t\}_{t = 1}^T \subset \RR^d$ for $t \in \{0, 1, \cdots, T\}$, define $\Zb_t = \lambda \Ib + \sum_{i = 1}^{t} \xb_i \xb_i^\top$. Then, provided that $\|\xb_t\|_2 \le M$ for all $t \in [T]$, we have \begin{align} 
        \sum_{t = 1}^T \min\{1, \|\xb_t\|_{\Zb_{t - 1}^{-1}}^2\} \le 2d \log \frac{d \lambda + T M^2}{d \lambda}. \notag
    \end{align} 
\end{lemma}

\begin{lemma} \label{lemma:potential2}
    For any $\lambda > 0$ and sequence $\{\xb_t\}_{t = 1}^T \subset \RR^d$ for $t \in \{0, 1, \cdots, T\}$, define $\Zb_t = \lambda \Ib + \sum_{i = 1}^{t} \xb_i \xb_i^\top$. Then, provided that $\|\xb_t\|_2 \le M$ for all $t \in [T]$, we have \begin{align} 
        \sum_{t = 1}^T \|\xb_t\|_{\Zb_{t}^{-1}}^2 \le 2d \log \frac{d \lambda + T M^2}{d \lambda}. \notag
    \end{align} 
\end{lemma}

\begin{proof} 
    Applying matrix inversion lemma, \begin{align} 
        \sum_{t = 1}^T \|\xb_t\|_{\Zb_{t}^{-1}}^2 &= \sum_{t = 1}^T \xb_t^\top \Zb_{t}^{-1} \xb_t \notag
        \\&= \sum_{t = 1}^T \xb_t^\top \left(\Zb_{t - 1}^{-1}- \frac{\Zb_{t - 1}^{-1} \xb_t \xb_t^\top \Zb_{t - 1}^{-1}}{1 + \xb_t^\top \Zb_{t - 1}^{-1} \xb_t} \right) \xb_t \notag
        \\&= \sum_{t = 1}^T \frac{\|\xb_t\|_{\Zb_{t - 1}^{-1}}^2}{1 + \|\xb_t\|_{\Zb_{t - 1}^{-1}}^2} \notag
        \\&\le\sum_{t = 1}^T \min\{1, \|\xb_t\|_{\Zb_{t - 1}^{-1}}^2\} \notag
        \\&\le 2d \log \frac{d \lambda + T M^2}{d \lambda}, \notag
    \end{align}
    where the second equality follows from matrix inversion lemma, the second inequality holds by Lemma \ref{lemma:potential}. 
\end{proof} 

\begin{lemma}[Concentration bound for sub-exponential random variables] \label{lemma:subexponential}
    Let $X$ be a sub-exponential random variable such that $X \sim \text{SE}(\sigma^2, \alpha)$. Then we have \begin{align} 
        \PP\left(X - \EE[X] \ge \beta\right) \le \begin{cases} 
            \exp\left(-\beta^2 / (2\sigma^2)\right),\quad 0 < \beta \le  \sigma^2 / \alpha \notag \\ 
            \exp\left(-\beta / 2\alpha\right),\quad t > \sigma^2 /\alpha \notag
        \end{cases}
    \end{align}
\end{lemma}

\begin{lemma}[Confidence Ellipsoid, Theorem 2, \citealt{abbasi2011improved}] \label{lemma:hoeffding}
        Let $\{\cG_k\}_{k=1}^\infty$ be a filtration, and $\{\xb_k,\eta_k\}_{k\ge 1}$ be a stochastic process such that
        $\xb_k \in \RR^d$ is $\cG_k$-measurable and $\eta_k \in \RR$ is $\cG_{k+1}$-measurable.
        Let $L,\sigma,\lambda, \epsilon>0$, $\bmu^*\in \RR^d$. 
        For $k\ge 1$, 
        let $y_k = \la \bmu^*, \xb_k\ra + \eta_k$ and
        suppose that $\eta_k, \xb_k$ also satisfy 
        \begin{align}
            \EE[\eta_k|\cG_k] = 0,\  \eta_k|\cG_k \sim \mathrm{subG}(R),\,\|\xb_k\|_2 \leq L.
        \end{align}
        For $k\ge 1$, let $\Zb_k = \lambda\Ib + \sum_{i=1}^{k} \xb_i\xb_i^\top$, $\bbb_k = \sum_{i=1}^{k}y_i\xb_i$, $\bmu_k = \Zb_k^{-1}\bbb_k$, and
        \begin{small}
        \begin{align}
            \beta_k &= R\sqrt{d \log \left(\frac{1 + kL^2 / \lambda}{\delta}\right)}.\notag
        \end{align}
        \end{small}
        Then, for any $0 <\delta<1$, we have with probability at least $1-\delta$ that, 
        \begin{align}
            \forall k\geq 1,\ \big\|\textstyle{\sum}_{i=1}^{k} \xb_i \eta_i\big\|_{\Zb_k^{-1}} \leq \beta_k,\ \|\bmu_k - \bmu^*\|_{\Zb_k} \leq \beta_k + \sqrt{\lambda}\|\bmu^*\|_2. \notag
        \end{align}
\end{lemma}

\begin{lemma}[\citealt{Pinsker1964InformationAI}] \label{lemma:pinsker}
    If $P$ and $Q$ are two probability distributions on a measurable space $(X, \bSigma)$, then for any measurable event $A \in \bSigma$, it holds that \begin{align*} 
        \left|P(A) - Q(A) \right| \le \sqrt{\frac{1}{2}\KL(P\|Q)} := \sqrt{\frac{1}{2}\EE_P\left(\log \frac{\text{d}P}{\text{d}Q}\right)}. 
    \end{align*}
\end{lemma}